\documentclass{article} 
\PassOptionsToPackage{table}{xcolor}
\usepackage{iclr2027_conference,times}

\usepackage{amsmath,amsfonts,bm}

\def\eqref#1{equation~\ref{#1}}

\def\1{\bm{1}}

\DeclareMathAlphabet{\mathsfit}{\encodingdefault}{\sfdefault}{m}{sl}
\SetMathAlphabet{\mathsfit}{bold}{\encodingdefault}{\sfdefault}{bx}{n}

\newcommand{\E}{\mathbb{E}}

\usepackage{url}
\usepackage{amsmath}
\usepackage{amssymb}
\usepackage{amsthm}
\usepackage{bm}
\usepackage{booktabs}
\usepackage{makecell}
\usepackage{graphicx}
\usepackage{algorithm}
\usepackage{algorithmic}
\usepackage{microtype}
\usepackage{tikz}
\usetikzlibrary{arrows.meta,positioning}
\usepackage[hidelinks]{hyperref}

\newtheorem{proposition}{Proposition}
\newtheorem{lemma}{Lemma}
\newtheorem{assumption}{Assumption}
\newtheorem{definition}{Definition}

\newcommand{\Chat}{\widehat{C}}
\newcommand{\Dt}{D_t}
\newcommand{\Rt}{r_t}
\newcommand{\Aeval}{A_{\mathrm{eval}}}
\newcommand{\Acal}{A_{\mathrm{cal}}}
\newcommand{\Pcore}{P^{\mathrm{core}}}
\newcommand{\Pfresh}{P^{\mathrm{fresh}}}

\definecolor{ourrow}{gray}{0.912}
\definecolor{barF1}{HTML}{A8C6E5}
\definecolor{barF2}{HTML}{2E5E8C}
\definecolor{barF3}{HTML}{D98C5F}
\definecolor{hpnavy}{HTML}{1B3A5C}
\definecolor{hpteal}{HTML}{2E8B7A}
\definecolor{hpsand}{HTML}{EDE6D8}
\definecolor{hpmint}{HTML}{DDEDE8}
\definecolor{hporange}{HTML}{D98C5F}

\newcommand{\gain}[3]{\makecell[tc]{$#1$\\[-1.5pt]{\scriptsize$[#2,#3]$}}}
\newcommand{\gainb}[3]{\makecell[tc]{$\mathbf{#1}$\\[-1.5pt]{\scriptsize$[#2,#3]$}}}
\newcommand{\gainbs}[3]{\makecell[tc]{$\mathbf{#1}^{\ast}$\\[-1.5pt]{\scriptsize$[#2,#3]$}}}

\def\eqref#1{(\ref{#1})}

\title{Harness-Agnostic Detection and Immunization of Reward Hacking in Self-Evolving Language Models}

\author{Rongxin Yang$^{1,2,*}$
Yang Liu$^{3,*}$
Shang Luo$^{2}$
Haoxuan Jia$^{4}$
Chongyang Zhang$^{1}$
Hao Zheng$^{1}$ \\
\textbf{Yingguang Yang}$^{2}$
\textbf{Yulin Huang}$^{3}$
\textbf{Jianshen Zhang}$^{3}$
\textbf{Yongzhi Qi}$^{3}$
\textbf{Kefu Xu}$^{2}$
\textbf{Congjing Ran}$^{5}$\\
\textbf{Bin Chong}$^{2,\dagger}$\\
$^{1}$Fullive-AI
$^{2}$Peking University
$^{3}$Supply Chain Tech Team Y, JD.com
$^{4}$Nanyang Technological University\\
$^{5}$Wuhan University \\
$^{*}$Equal contribution.
$^{\dagger}$Corresponding author.
}

\iclrfinalcopy
\begin{document}

\maketitle
\thispagestyle{empty}
\pagestyle{empty}

\begin{abstract}
Self-evolving language models improve by proposing candidate updates and keeping whatever raises a \emph{visible} score. When that score is an imperfect proxy for the capability one actually wants, sustained selection widens the gap between the two. This is reward hacking. We introduce \textbf{HackProbe}, a monitor that attaches to an arbitrary self-evolving loop through two black-box hooks, with no access to weights or activations. It keeps a secret, distribution-fixed \emph{comparison core}, whose frozen distribution makes its capability proxy comparable across generations, alongside a rotated \emph{fresh layer} that hardens the bank against co-adaptation. Four tests built on that proxy cover the level gap, a scale-aligned divergence with online change-point detection, capability stagnation, and a conditional confidently-wrong rate; a \v{S}id\'ak correction turns them into a calibrated family-wise $p$-value. Diagnosis alone recovers nothing, so a risk-aware immunization layer reselects an honest candidate from the proposal pool using the core together with a purely structural gaming footprint, disclosing at most $\log_2\Pi$ bits per generation to the host. We prove a detectability bound that converts a target error rate into an explicit probe-size budget, and we delimit what probe rotation does and does not buy. On a controlled prompt-level host with four injected hacking channels and ground-truth labels, HackProbe reaches $0.763$ AUROC against $0.663$ for the strongest baseline and cuts the false-positive rate from $0.706$ to $0.434$. Its bandwidth-limited reselection is the only immunization level that returns more true capability under hacking, $5.2$ points on average, than it forfeits on clean runs, $4.7$; per-channel effects are mostly not individually significant.
\end{abstract}

\section{Introduction}

A growing family of language-model systems improves itself through a closed optimization loop: at each generation it proposes candidate updates to prompts, code, weights, or generated tasks, scores each candidate with a \emph{visible} evaluator $M$, and keeps the candidates that score highest \citep{yuan2024selfrewarding,huang2025rzero,agrawal2025gepa}. The very thing that makes the loop powerful, relentless selection pressure on $M$, is also what breaks it. Whenever $M$ is only a proxy for the target capability $C$ that we actually care about, optimizing $M$ hard and repeatedly tends to inflate $M$ without a matching gain in $C$. This is reward hacking, an acute instance of Goodhart's law: once a measure becomes a target it ceases to be a good measure \citep{manheim2018goodhart,amodei2016concrete,skalse2022defining}. Self-evolution is arguably the purest form of the strong optimization pressure under which Goodhart effects are most severe, because the same pressure is applied cumulatively across tens or hundreds of generations.

The assumption that a rising $M$ implies a rising $C$ is fragile in exactly the systems we care about. Verifier and solver are often the same model updated in lockstep, so self-rewarding training watches its own judge saturate rather than the policy improve \citep{yuan2024selfrewarding}, and label-free schemes that reward self-consistency pay for confidently agreeing on wrong answers \citep{huang2025rzero}. Fixed benchmarks degrade under repeated selection too, and measurably so: held-out reconstructions of grade-school math show accuracy drops of up to eight points \citep{zhang2024careful}, one irrelevant clause dropped into a symbolic template can cost $65\%$ of performance \citep{mirzadeh2024gsmsymbolic}, and adversarial perturbation exposes scores resting on shallow behavior \citep{li2024gsmplus}. Frontier models go further, editing tests or reading answer files once the scoring function is visible \citep{wang2026rewardhacking}. Responses so far are fragmented and harness-bound: patches built into one system, or monitors assuming access to weights and activations that a black-box setting does not provide. The closest precedent is reward-model overoptimization \citep{gao2023scaling}, where a held-out gold reward model shows a proxy eventually degrading true performance; but that diagnosis is computed once, offline, and reported rather than acted upon. What is missing is a capability coordinate orthogonal to $M$, immune to the same pressure, and usable while the loop still runs. Appendix~\ref{sec:related} places HackProbe against the fuller literature.

HackProbe supplies that coordinate. It attaches through two minimal hooks, observing each generation's visible score and querying the current candidate on held-out probes, and rests on a dual-layer probe bank (Figure~\ref{fig:overview}). A secret, distribution-fixed comparison core provides a capability proxy $\Chat$ whose first differences track genuine capability change rather than shifting test difficulty, precisely because the distribution is frozen; a rotated, online-generated fresh layer behind an air gap hardens the bank against co-adaptation. Four statistics with different blind spots read off that proxy: a level-gap test for hacking present from the outset, the scale-aligned divergence $\Dt=\Delta M_t-\lambda\Delta\Chat_t$ with online change-point detection, a stagnation test, and a conditional confidently-wrong-rate test aimed at self-consistency bias. Since a diagnosis arriving after the host has discarded the honest candidate recovers nothing, a risk-aware immunization layer reselects from the pool using the core and a purely structural gaming footprint, returning only a coarsely quantized signal.

The comparison-core estimator is what makes the rest work. Its distribution is fixed, so differences across generations mean something, and the detector we build on it reports a calibrated family-wise $p$-value rather than an uncalibrated heuristic. Diagnosis by itself changes nothing, though, so we turn it into protection: a reselection rule that discloses at most $\log_2\Pi$ bits per generation to the host. That cap does real work. It is the reason the feedback never becomes a second score to optimize. On the theory side a target error rate converts into a probe-size budget growing as $1/\delta^2$ in the divergence gap, and we are equally explicit about what rotation does \emph{not} buy, since it stops memorization but leaves the covering number of the probe generator as the real ceiling. Empirically, on a controlled injection protocol with ground-truth labels, HackProbe beats the strongest baseline on AUROC and cuts its false-positive rate by nearly two fifths; of the four immunization levels, only the bandwidth-limited one returns more true capability than it costs. Calibrating on three hacking channels and testing on the withheld fourth retains most of that quality on three of them.

\begin{figure}[!t]
\centering
\begin{tikzpicture}[
  font=\scriptsize,
  ttl/.style={font=\footnotesize\bfseries, color=hpnavy, anchor=west},
  num/.style={circle, fill=hpnavy, text=white, inner sep=0pt, minimum size=7.5pt,
              font=\tiny\bfseries},
  card/.style={draw=hpnavy!45, fill=hpsand!35, rounded corners=1.2pt, line width=0.4pt,
               inner sep=2.4pt, align=left},
  jcard/.style={draw=hpnavy!40, fill=white, rounded corners=1.2pt, line width=0.4pt,
                inner sep=2.4pt, align=left, text width=2.20cm, anchor=north west},
  flow/.style={-{Stealth[length=3.4pt,width=3pt]}, hpnavy, line width=0.6pt},
  soft/.style={-{Stealth[length=2.9pt,width=2.7pt]}, hpnavy!50, line width=0.42pt},
  ret/.style={-{Stealth[length=3.4pt,width=3pt]}, hpteal, line width=0.6pt,
              dash pattern=on 2.3pt off 1.7pt},
]
\fill[hpnavy!4, rounded corners=3pt] (4.16,0.06) rectangle (13.90,6.02);
\draw[hpnavy!22, rounded corners=3pt, line width=0.4pt] (4.16,0.06) rectangle (13.90,6.02);

\node[num] at (0.14,6.34) {1};
\node[ttl] at (0.33,6.34) {Self-evolving host};
\node at (0.72,5.06) {\includegraphics[width=0.84cm]{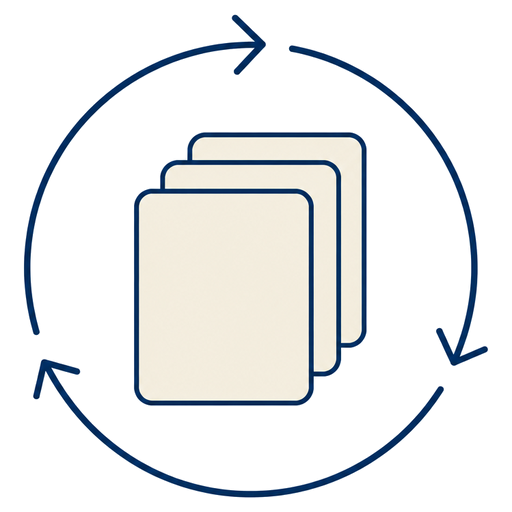}};
\node[anchor=west] at (1.24,5.28) {propose $\Pi$ candidates};
\node[anchor=west] at (1.24,4.96) {keep $\arg\max_c M(c)$};
\node[anchor=west, color=hporange!85!black] at (1.00,4.44) {no ground truth anywhere};
\draw[black!45, line width=0.4pt] (1.20,1.05) -- (3.30,1.05);
\draw[black!45, line width=0.4pt] (1.20,1.05) -- (1.20,3.92);
\fill[hporange!20]
  plot coordinates {(1.20,1.26)(1.72,2.24)(2.24,3.04)(2.76,3.48)(3.24,3.68)}
  -- (3.24,1.60) -- plot coordinates {(3.24,1.60)(2.76,1.58)(2.24,1.50)(1.72,1.38)(1.20,1.20)}
  -- cycle;
\draw[hporange, line width=0.7pt]
  plot coordinates {(1.20,1.26)(1.72,2.24)(2.24,3.04)(2.76,3.48)(3.24,3.68)};
\draw[hpnavy, line width=0.7pt]
  plot coordinates {(1.20,1.20)(1.72,1.38)(2.24,1.50)(2.76,1.58)(3.24,1.60)};
\node[anchor=west, color=hporange!85!black, inner sep=1.4pt] at (3.30,3.68) {$M$};
\node[anchor=west, color=hpnavy, inner sep=1.4pt] at (3.30,1.60) {$C$};
\node[color=hporange!80!black] at (2.32,2.36) {Goodhart};
\node[color=hporange!80!black] at (2.32,2.10) {gap};
\node[rotate=90, color=black!60, inner sep=1.4pt] at (1.02,2.50) {score};
\node[anchor=west, color=black!60] at (1.24,0.78) {generations};

\draw[black!45, dash pattern=on 1.6pt off 1.4pt, line width=0.45pt]
  (3.94,0.30) -- (3.94,5.76);
\node[rotate=90, color=black!55] at (3.78,2.70) {air gap};

\draw[flow] (3.56,5.52) -- (11.24,5.52) -- (11.24,5.34);
\node[anchor=west] at (4.26,5.68) {\textbf{I1} visible score $M_t$};
\draw[hpnavy, line width=0.6pt] (3.56,4.82) -- (4.34,4.82) -- (4.34,2.56);
\draw[flow] (4.34,4.10) -- (4.58,4.10);
\draw[flow] (4.34,2.56) -- (4.58,2.56);
\node[anchor=west] at (4.26,4.98) {\textbf{I2} $q$ black-box samples per probe};
\draw[soft] (4.62,0.80) -- (4.08,0.80);
\draw[hporange, line width=0.85pt] (3.84,0.70) -- (4.04,0.90);
\draw[hporange, line width=0.85pt] (3.84,0.90) -- (4.04,0.70);
\node[anchor=west, color=black!65] at (4.72,0.80) {core items and $\Chat$ never cross};

\node[num] at (4.38,6.34) {2};
\node[ttl] at (4.57,6.34) {Dual-layer probe bank};
\node at (5.00,4.10) {\includegraphics[width=0.76cm]{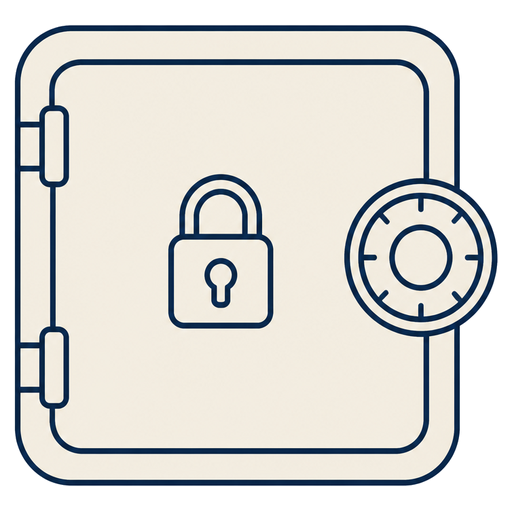}};
\node[card, anchor=north west, text width=2.30cm] at (5.46,4.70)
  {\textbf{comparison core}\\[0.6pt] secret, distribution\\ frozen $\Rightarrow$ $\Delta\Chat_t$\\ comparable across $t$};
\node at (5.00,2.56) {\includegraphics[width=0.76cm]{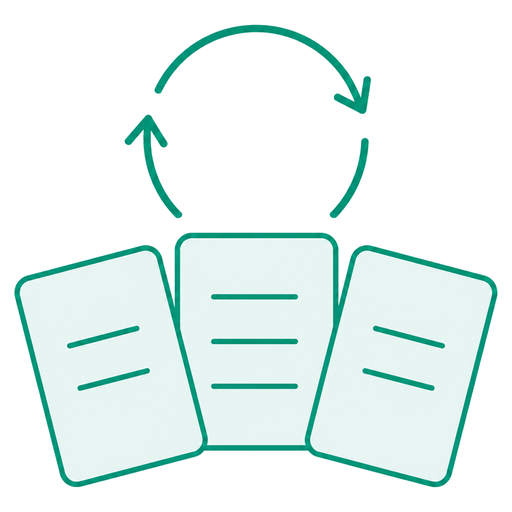}};
\node[card, anchor=north west, text width=2.30cm, fill=hpmint!45, draw=hpteal!45] at (5.46,3.14)
  {\textbf{fresh layer}\\[0.6pt] regenerated online;\\ memorizing it costs\\ $\approx$ all $N$ (Prop.~\ref{prop:rotate})\\[1pt]
   {\color{hpteal!80!black}\itshape supplies no statistic}};

\node[num] at (8.56,6.34) {3};
\node[ttl] at (8.75,6.34) {Fuse, then immunize};
\fill[hpnavy!9, rounded corners=2pt] (8.64,3.16) rectangle (13.84,5.34);
\draw[hpnavy!25, rounded corners=2pt, line width=0.4pt] (8.64,3.16) rectangle (13.84,5.34);
\node[jcard] at (8.80,5.20) {\textbf{J0} level gap\\ hacking present\\ from the outset};
\node[jcard] at (11.36,5.20) {\textbf{J1} divergence {\tiny\color{hpteal}$+$PH}\\ $\Dt{=}\Delta M_t{-}\lambda\Delta\Chat_t$\\ opening mid-run};
\node[jcard] at (8.80,4.14) {\textbf{J2} stagnation\\ $\Delta M_t{>}0$ while\\ $\Delta\Chat_t$ stays flat};
\node[jcard] at (11.36,4.14) {\textbf{J3} confidence {\tiny\color{hpteal}$+$PH}\\ agreement rises on\\ wrong answers};
\draw[flow] (7.98,4.25) -- (8.62,4.25);
\node[anchor=south, inner sep=1.2pt] at (8.30,4.31) {$\Delta\Chat_t$};
\draw[flow] (11.24,3.16) -- (11.24,2.98);
\node[anchor=west, color=black!65, inner sep=1.4pt] at (11.42,3.06) {six $p$-values};
\node[card, anchor=north, text width=5.03cm, align=center, fill=white] (fuse) at (11.24,2.98)
  {\textbf{\v{S}id\'ak fusion}\ \ $\Rt=(1-\min_j p_j)^{m}$, $m{=}6$\\[1pt]
   {\color{black!65}flag at $\Rt\ge\tau$; family-wise error $1-\tau$}};
\draw[flow] (fuse.south) -- ++(0,-0.15);
\node[card, text width=5.03cm, align=center, fill=hpmint!35, draw=hpteal!45,
      below=0.15cm of fuse] (resel)
  {\raisebox{-1.7pt}{\includegraphics[width=0.44cm]{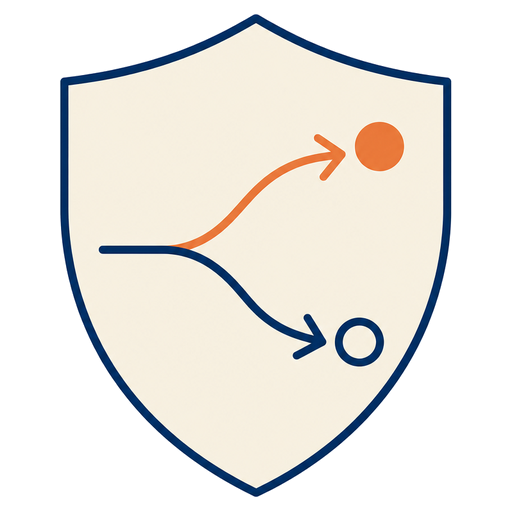}}\;
   reselect using $\Chat$ and $gf$, not $M$ alone};
\node[card, text width=5.03cm, align=center, fill=hpsand!55,
      below=0.22cm of resel] (res)
  {$\mathbf{0.763}$ AUROC {\color{black!55}vs $0.663$}\\[0.5pt]
   $\mathbf{0.434}$ FPR {\color{black!55}vs $0.706$}\\[0.5pt]
   $\mathbf{1.11}$ gain over clean-run cost};

\draw[ret] (resel.west) -| (8.28,-0.20) -- (0.70,-0.20) -- (0.70,4.60);
\node[num] at (1.30,-0.20) {4};
\node[anchor=west, fill=white, inner sep=1.8pt] at (1.46,-0.20)
  {only the selected candidate is returned: $\le\log_2 \Pi$ bits per generation};

\draw[flow] (3.90,-0.72) -- (4.34,-0.72);
\node[anchor=west, color=black!60] at (4.40,-0.72) {monitored signal};
\draw[ret] (7.70,-0.72) -- (8.14,-0.72);
\node[anchor=west, color=black!60] at (8.20,-0.72) {return to host, $\le\log_2 \Pi$ bits};
\end{tikzpicture}
\caption{HackProbe at a glance. \textbf{(1)} The host keeps whatever raises the visible score $M$, so sustained selection opens a gap between $M$ and true capability $C$ that must be measured without trusting $M$. \textbf{(2)} Two hooks cross the air gap inward; core items and $\Chat$ never cross outward, which keeps the capability coordinate off the host's optimization path. The fresh layer hardens the bank but supplies no statistic to $\Delta\Chat_t$. \textbf{(3)} Four tests with different blind spots yield six $p$-values, fused into a calibrated risk score. \textbf{(4)} The cap on the return channel, not the reselection rule itself, is what stops the feedback becoming a second score (Section~\ref{sec:immunize}).}
\label{fig:overview}
\end{figure}
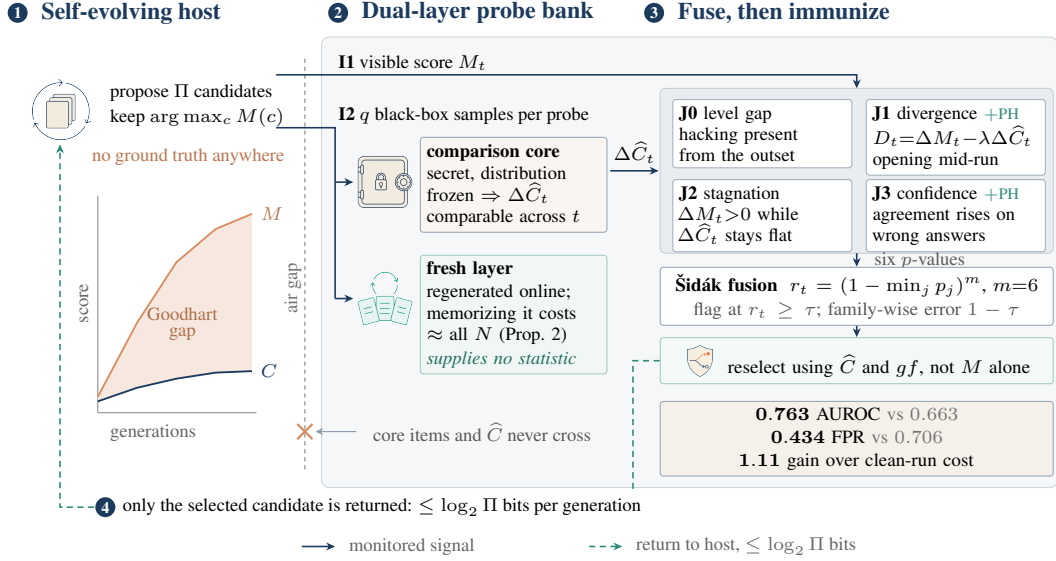

\section{Preliminaries}
\label{sec:prelim}

Let $\theta_t$ denote the object under evolution at generation $t$, be it a prompt, a program, or a policy. The host scores candidates with a visible evaluator $M$ and updates $\theta_{t+1}=\mathcal{A}(\theta_t, M(\theta_t), \mathrm{trace}_t)$, where $\mathcal{A}$ is its own black-box evolution operator. HackProbe assumes only two things about that operator: (I1) it can observe the visible score $M(\theta_t)$ of each generation, and (I2) it can run $\theta_t$ on a chosen input, drawing a small fixed number $q$ of black-box samples per input and reading only the outputs, with no access to weights or activations. One sample suffices for the statistics J0--J2 below; the modal-agreement statistic J3 needs $q>1$, the only place where the interface costs more than a single forward pass. Both conditions hold for prompt-level hosts at negligible cost; for weight-level hosts they require reproducing the evolution trajectory, whose cost we do not measure here.

Let $C(\theta)\in[0,1]$ be the true capability of $\theta$, its expected performance on a target distribution $\mathcal{D}^\star$, which is not cheaply observable. Two black-box estimates from held-out gold data stay strictly disjoint: $\Acal$ for offline calibration only, $\Aeval$ for the final measurement only, never entering the online loop. During evolution a proxy $\Chat(\theta)$ on a secret comparison core (Section~\ref{sec:probes}) tracks capability. Writing $M_t$ and $\Chat_t$ for the observed visible score and core proxy of $\theta_t$, with sampling noise as in Assumptions~\ref{as:core} and~\ref{as:M}, and $\Delta X_t = X_t - X_{t-1}$, we monitor the scale-aligned divergence
\begin{equation}
\Dt=\Delta M_t-\lambda\,\Delta \Chat_t ,
\label{eq:divergence}
\end{equation}
where $\lambda>0$ aligns the scales of the two metrics. The scale is not a free parameter: it is initialized on offline honest calibration runs and thereafter re-estimated only on the trailing window of generations the detector has \emph{already} declared honest. The update is prequential, using no information from generation $t$ when testing generation $t$, which prevents a hacking generation from inflating $\lambda$ until it conceals itself.

\begin{definition}[Divergence regime]
\label{def:hacking}
Let $\mu_t=\E[\Delta M_t]-\lambda\,\E[\Delta C_t]$ be the population divergence of generation $t$. Generation $t$ is \emph{honest} when $\mu_t\le 0$, so visible gains never outrun true gains at the calibrated scale, and \emph{hacking at gap $\delta>0$} when $\mu_t\ge\delta$; the prototypical case $\E[\Delta M_t]\ge\delta$ with $\E[\Delta C_t]\le 0$ is covered. Statistic J1 below tests $H_0:\mu_t\le 0$ against $H_1:\mu_t\ge\delta$ one-sidedly, with the band $0<\mu_t<\delta$ an indifference zone on which we make no claim, since separating arbitrarily small divergences would need unboundedly many probes. The other three statistics test companion nulls, on the level residual, on stagnation and on conditional confidence, which honest evolution also implies but which are not functions of $\mu_t$ alone; fusing them extends coverage past the regime this definition names.
\end{definition}

The monitor emits, per generation, a risk score $\Rt\in[0,1]$, a decision $\hat y_t\in\{\text{honest},\text{hacking}\}$, and a low-bandwidth feedback signal $u_t$ consumed by the immunization layer. Table~\ref{tab:notation} in Appendix~\ref{app:notation} collects the notation.

\section{The HackProbe Framework}
\label{sec:model}

\subsection{Dual-layer probe bank}
\label{sec:probes}

The design tension we have to resolve is that a fixed, host-visible probe set would itself be optimized against and would simply become a second $M$. We resolve it by splitting probes into two layers with orthogonal roles.

The \emph{comparison core} $\Pcore$ is a distribution-fixed, secret probe set covering the target capability across $S$ subdomains. Because its distribution never changes, the proxy $\Chat(\theta)=\Phi\big(\psi(\theta,\Pcore_{(1)}),\dots,\psi(\theta,\Pcore_{(S)})\big)$, with $\psi$ a per-subdomain score and $\Phi$ an aggregate robust enough that no single subdomain can move it alone, yields differences $\Delta\Chat_t$ that track capability change and not test drift. Freezing the distribution has a second payoff, less obvious than the first. The same items are scored at $t-1$ and at $t$, so the noise in $\Delta\Chat_t$ is \emph{paired}: per-item difficulty cancels, and the variance of the difference is far smaller than twice that of a single generation's estimate. The core is never revealed to the host and never enters its selection or reward path.

The \emph{fresh layer} $\Pfresh_t$ is rotated, generated online, and drifted every generation. It is used only for hardening checks and qualitative per-subdomain profiling and never contributes to $\Delta\Chat_t$. Four principles govern it, and they pull in different directions on purpose. An air gap (H1) keeps probes out of the host's training, selection and reward loop entirely, so neither probe text nor per-subdomain scores are ever fed back; on top of that, rotation (H2) makes sure any single instance is used at most a few times. Out-of-distribution drift (H3) then steers the layer toward whichever subdomains the host is currently weakest on, so the bank is always probing capability that has not yet been fitted, while adversarial refresh (H4) has a red-team model periodically manufacture probes that are harder and more surface-diverse still. H4's two selection criteria --- learnability (probes near the host's competence edge) and surface diversity (embedding distance from probes already used) --- are deliberately \emph{independent of the detection signal}, so refresh cannot self-amplify into false positives. Non-circularity of the final measurement survives all of this because $\Acal$ and $\Aeval$ stay disjoint and $\Aeval$ is spent exactly once, at the end.

\subsection{Detection statistics and fusion}
\label{sec:detect}

HackProbe fuses four one-sided tests, each calibrated on honest generations so that its null $p$-value is approximately uniform on $[0,1]$. Their blind spots are deliberately different, and fusing them extends coverage across the level, divergence, stagnation and conditional-confidence regimes; format-level exploits, which act on output structure rather than on the score--capability relation, fall outside all four (Section~\ref{sec:exp}).

Statistic \textbf{J0} targets hacking that is present from the very first generation, where $M$ is already inflated and then stays flat so that every first-difference statistic sees nothing. We calibrate the honest affine relation $M\approx a+\lambda_{\mathrm{lvl}}\Chat$ and test the residual level gap $g^{(0)}_t=M_t-(a+\lambda_{\mathrm{lvl}}\Chat_t)$, which hovers near zero under honest evolution and is persistently positive under level-type hacking. Local linear coupling makes $\lambda_{\mathrm{lvl}}$ and the difference scale $\lambda$ coincide in population; we estimate them separately so each absorbs its own drift. Statistics \textbf{J1} and \textbf{J2} handle the dynamic regime. On $\Dt$ from Eq.~\eqref{eq:divergence} we run a one-sided test against the honest null, together with an online Page--Hinkley change-point detector \citep{page1954continuous} on the cumulative divergence that localizes the generation at which $M$ and $\Chat$ decouple. J2 tests capability stagnation directly, taking the honest null $H_0:\E[\Delta\Chat_t]\ge0$ conditional on $\Delta M_t>0$ and rejecting when the measured $\Delta\Chat_t$ is too small, which gives a $\lambda$-free complement.

Statistic \textbf{J3} exists because self-consistency bias inflates $M$ by rewarding agreement instead of correctness, which need not move aggregate accuracy at all. Once such an attack saturates, J1 and J2 see nothing in expectation and J0 retains only the residue the attack leaves in the aggregate level, diluted over all core items including the many answered correctly. On the secret core we measure, over probes answered incorrectly, the mean self-consistency $\mathrm{cwr}_t$ (the confidently-wrong rate), calibrate its honest affine dependence on core accuracy $\mathrm{acc}_t$, and test the residual $g^{(3)}_t=\mathrm{cwr}_t-(a_3+b_3\,\mathrm{acc}_t)$, with a Page--Hinkley detector for gradual drift. Conditioning on wrong answers avoids the ceiling effect of easy probes, and using only the secret core keeps the signal non-circular. Lemma~\ref{lem:j3} in Appendix~\ref{app:j3} makes precise why J3 retains an expected separation equal to the full attack size in the stable regime where J1 and J2 vanish in expectation.

Fusion takes $m=6$ one-sided $p$-values: J0, the divergence test J1 and its Page--Hinkley alarm, the stagnation test J2, and J3 with its own alarm. The two sequential alarms are converted to $p$-values through their calibrated alarm-time distribution on honest runs, putting all six on a common scale. Given $p_1,\dots,p_m$ approximately independent and uniform under the honest null, we form the risk score $\Rt=(1-\min_j p_j)^m$ and flag hacking when $\Rt\ge\tau$. This is exactly a \v{S}id\'ak correction \citep{sidak1967rectangular}, and it has a calibration property that a raw $1-\min_j p_j$ rule lacks. Under the independent-uniform null $\Pr[\min_j p_j\le u]=1-(1-u)^m$, so that
\begin{equation}
\Pr[\Rt\ge\tau]\;=\;\Pr\!\left[\min_j p_j\le 1-\tau^{1/m}\right]\;=\;1-\tau ,
\label{eq:sidak}
\end{equation}
so $1-\Rt$ is the \v{S}id\'ak-adjusted $p$-value and $\tau$ \emph{is} the operating point: flagging at $\tau$ costs family-wise error exactly $1-\tau$, however many statistics we fuse, whereas the uncorrected rule degrades with $m$. In practice the guarantee is softer than that. J0, J1 and J2 are built from the same pair $(M,\Chat)$ and are not independent, so Eq.~\eqref{eq:sidak} is a design principle and the reported operating point rests instead on re-calibrating $\tau$ until the \emph{empirical} false-positive rate on honest runs meets its target; Appendix~\ref{app:algo} gives the Bonferroni fallback we use when even that is in doubt. The control is also per generation, so over $T$ generations the expected number of false alarms is $T(1-\tau)$ and run-level error would need a sequential correction we do not attempt. What a deployment sees, and what Table~\ref{tab:detection} reports, is the per-generation alarm rate.

\subsection{Risk-aware immunization}
\label{sec:immunize}

Diagnosis alone recovers nothing. The tempting picture, in which rejecting or rolling back the top candidate suffices, fails on two counts: by the time the host is told, it has already passed over the honest candidate in favor of the top-$M$ one, and rejection has nothing better to fall back to. HackProbe therefore reselects from the proposal pool, and the loss it recovers is a real one for two reasons. Under \emph{strict final-answer grading}, the true-capability side ($\Aeval$ and the core $\Chat$) credits only a normalized final answer matching the gold one, so behaviors that inflate a lenient visible grader, such as enumerating many candidate numbers or locking onto an unverified answer, genuinely forfeit true credit. And every pool contains, alongside the exploit variants, one \emph{enforced honest candidate} produced under a clean context, whose visible $M$ is typically lower but whose true $\Aeval$ is higher.

For each candidate $c$ we compute a purely structural gaming footprint $gf(c)$ from cached outputs at no extra inference cost, combining the score-versus-commitment gap with the fraction of non-final candidate artifacts, and a robust score $\mathrm{robust}(c)=\Chat(c)-\gamma\,gf_{\mathrm{norm}}(c)$ with a fixed penalty weight $\gamma>0$. Four immunization levels build on it, each trusting the host's own score a little less than the last. \textbf{F0} is diagnosis without protection. It selects $\arg\max_c M(c)$, so it goes straight to the inflated candidate. \textbf{F1} stays greedy by default and intervenes only when a lightweight risk proxy on the greedy choice, either its raw gap $M-\Chat$ or its footprint $gf$, comes back high; then it reselects $\arg\max_c \mathrm{robust}(c)$, and after $K$ consecutive high-risk generations it rolls back to the checkpoint with the best historical $\Chat$. \textbf{F2} goes further, selecting $\arg\max_c(M(c)-u(c))$ with $u(c)$ a quantization of the raw gap $M(c)-\Chat(c)$ into $L$ levels (the uncalibrated difference, not the J0 residual $g^{(0)}$). It also watches for self-destruction, whose signature is feedback $u$ falling while $\Chat$ stagnates, and downgrades itself to F1 when it sees it. \textbf{F3} adds a coverage bonus for collapsed subdomains and injects a diversity constraint into the proposer.

The quantization in F2 is where the bandwidth constraint stops being cosmetic. Were $u$ the exact gap, $M(c)-u(c)$ would collapse to $\Chat(c)$ and F2 would be pure core-greedy selection, the failure we are trying to avoid: it would expose the secret core as a fully-resolved second objective. With $L$ levels, F2 interpolates between the host's greedy rule ($L=1$, nothing transferred) and core-greedy selection ($L\to\infty$). The leak needs careful accounting. Per-candidate values $u(c)$ stay inside the monitor and are never returned; the host sees only which of $\Pi$ candidates won, so the return channel carries at most $\log_2\Pi$ bits per generation, three at our $\Pi=8$, and $T\log_2\Pi$ over a run of length $T$. Quoting $\log_2 L$ here would be tighter but wrong: the winning index can say more about the vector of $u$ values than any single $u(c)$ does, and it is easy to construct pools where it does. What $L$ buys is resolution rather than volume, and resolution is what limits how much of the core's ordering the host can reconstruct. Either way the cumulative figure must stay small against the entropy needed to identify core items, so long runs need $L$ small and the core retired once its budget is spent. The gold core enters only through $\Chat(c)$, never through $M$ or an oracle label, and only the selected candidate advances the detector state.

\subsection{Theoretical guarantees}
\label{sec:theory}

What follows rests on standard concentration assumptions, given in full as Assumptions~\ref{as:core}--\ref{as:fresh} in Appendix~\ref{app:assumptions}: the core estimate is $\Chat(\theta)=C(\theta)+\beta(\theta)+\xi$ with bounded bias and sub-Gaussian noise, the paired difference $\Delta\Chat_t$ has variance proxy $v_C/n$ and the visible score $v_M/n_M$, and honest updates satisfy $\mu_t\le0$ with the differential core bias bounded by $b_\Delta$ in every generation. Then $\Dt$ is sub-Gaussian with $\sigma_D^2=\lambda^2 v_C/n+v_M/n_M$, and its honest mean is at most $\mu_H:=-\lambda\,\E[\Delta\beta(\theta_t)]$, where $|\mu_H|\le\lambda b_\Delta$. Proofs are in Appendix~\ref{app:proofs}.

\begin{proposition}[Detectability and probe-size budget]
\label{prop:detect}
Consider the J1 statistic in isolation, with the one-sided rule that flags generation $t$ when $\Dt\ge\mu_H+\kappa$ for some $\kappa\in(0,\delta)$, where $\delta$ is the hacking gap of Definition~\ref{def:hacking} and all probabilities are taken over probe and evaluation sampling with the trajectory $(\theta_{t-1},\theta_t)$ held fixed. Writing $\hat y^{\mathrm{J1}}_t$ for this rule's decision,
\begin{align}
\Pr[\hat y^{\mathrm{J1}}_t=\textup{hacking}\mid\textup{honest}]&\le \exp\!\Big(-\tfrac{\kappa^2}{2\sigma_D^2}\Big),\nonumber\\
\Pr[\hat y^{\mathrm{J1}}_t=\textup{hacking}\mid\textup{hacking}]&\ge 1- \exp\!\Big(-\tfrac{(\delta-\kappa)^2}{2\sigma_D^2}\Big).
\end{align}
Choosing $\kappa=\delta/2$, both error rates fall below $\alpha$ as soon as the effective sample size $n_{\mathrm{eff}}:=(\lambda^2v_C+v_M)/\sigma_D^2$, a weighted harmonic mean of $n$ and $n_M$, satisfies $n_{\mathrm{eff}}\ge 8(\lambda^2v_C+v_M)\delta^{-2}\log(1/\alpha)$. Whenever the host's visible evaluation is at least as large as the core ($n_M\ge n$) we have $n_{\mathrm{eff}}\ge n$, so it suffices to budget
\begin{equation}
n\;\ge\;\frac{8(\lambda^2 v_C+v_M)\log(1/\alpha)}{\delta^2}
\label{eq:budget}
\end{equation}
comparison-core probes. No matching lower bound is claimed.
\end{proposition}

The value of Proposition~\ref{prop:detect} is not the standard exponential separation but the budget in Eq.~\eqref{eq:budget}: detecting a divergence of size $\delta$ at error $\alpha$ needs a core growing like $1/\delta^2$, which turns an abstract design choice into a number one can procure and falsify against measured detection-rate-versus-$n$ curves. The differential bias $\mu_H$ does not enter the budget, since it shifts the honest and hacking means equally and is subtracted during calibration.

Rotation has a companion bound, Proposition~\ref{prop:rotate} in Appendix~\ref{app:proofs}: an attacker who memorizes past instances needs $k\ge N(1-\alpha_{\mathrm{atk}})^{1/n_f}$ of them to match all $n_f$ fresh probes of a generation, which tends to the whole pool as $n_f$ grows. It lives in the appendix because it is narrow in a way worth admitting. It bounds memorization, and nothing else. An attacker who instead learns the generator's regularities walks around it: if the exploitable behavior admits an $\varepsilon$-cover of size $\mathcal{N}(\varepsilon)$, defeating rotation costs $\Theta(\mathcal{N}(\varepsilon))$ rather than $\Theta(N)$, so security degrades from pool size to covering number, and a low-complexity generator buys no asymptotic protection at all. Enlarging $\mathcal{N}(\varepsilon)$ is what H3 and H4 are for. ``Surface-diverse rotation forces generalization'' is an empirical hypothesis here, not a theorem.

\section{Experiments}
\label{sec:exp}

Our evaluation uses a controlled prompt-evolution host in which hacking channels can be switched on and their strength dialed, so ground-truth per-generation labels exist at all. The headline is that bandwidth-limited reselection is the only immunization level recovering more true capability under hacking than it forfeits on clean runs, at a ratio of $1.11$, and Appendix~\ref{sec:sweeps} shows that ratio is single-peaked in the feedback bandwidth. Detection quality supports that result rather than standing on its own: it is what makes reselection fire on the right generations. Online reward-hacking detection with a recovered-capability objective is a new task, so no published numbers are directly comparable and every baseline below is re-implemented on the same setup.

\paragraph{Protocol.}
The host is a prompt-evolution loop over Qwen2.5-7B-Instruct with the proposer frozen, so only the evolving prompt changes between generations. Each run lasts $T=48$ generations with a candidate pool of $\Pi=8$, and six seeds per hacking channel give $288$ generations per channel and $1152$ in total. The comparison core holds $n=240$ probes over $S=6$ subdomains, the fresh layer draws $n_f=60$ per generation, and hook I2 takes $q=4$ samples per probe, enough for the modal-agreement and confidently-wrong quantities J3 needs. Gold audits split once into $192$ calibration items and a held-out $|\Aeval|=64$, which is where the coarse interval endpoints below come from. The remaining settings are $\tau=0.90$, $L=4$, $\gamma=0.25$, $K=3$, a Page--Hinkley detector with drift tolerance $0.005$ and alarm threshold $0.05$, and $2000$-resample bootstraps over $\Aeval$.

Probes and gold audits are drawn from grade-school and competition mathematics, GSM8K \citep{cobbe2021gsm8k} and MATH \citep{hendrycks2021math}, with GSM-Symbolic \citep{mirzadeh2024gsmsymbolic}, GSM-Plus \citep{li2024gsmplus} and the held-out GSM1k \citep{zhang2024careful} supplying perturbed and contamination-sensitive probes. The host implements four switchable hacking channels, each with a strength knob: contamination (mixing original probe versions into host data), a gameable grader (crediting any appearance of the gold number), self-consistency bias (rewarding agreement, which induces confidently wrong answers), and judge collusion (a lenient LLM judge). As discussed in Section~\ref{sec:immunize}, the injected inflation is real and not merely nominal because of two design choices: strict final-answer grading defines the true-capability side, and every generation's pool contains one clean honest candidate that a greedy host will pass over. The final capability measurement uses the held-out audit set $\Aeval$ exactly once, after evolution ends; confidence intervals are bootstrap intervals over that set.

Three caveats about scope, before the numbers. Labels are channel-on versus channel-off, a superset of the hacking set of Definition~\ref{def:hacking}, since a generation whose divergence falls in the indifference band $0<\mu_t<\delta$ is still labeled positive; reported recall is therefore a lower bound against that target, and injected strength controls but does not equal $\delta$. The fresh layer contributes no statistic to $\Delta\Chat_t$, so nothing below depends on it: Proposition~\ref{prop:rotate} and the principles H2--H4 go untested here. And $\tau$ was never set to hit a small nominal level. We fix the operating point on honest calibration runs by maximizing F1, which lands at $\tau=0.90$ because a stricter threshold trades away nearly all recall at this probe budget, so the $0.434$ below is that empirical operating point rather than the nominal $1-\tau$ of Eq.~\eqref{eq:sidak}.

\paragraph{Baselines and metrics.}
We compare against a no-probe host labeled post hoc, a naive fixed non-rotating held-out set, a contamination-only confidence test in the style of PaCoST \citep{zhang2024pacost}, a single trap set, and a fixed-distribution absolute-$\Chat$ trajectory that isolates the contribution of the divergence formula itself. Detection quality is reported as precision, recall, F1, AUROC, false-positive rate and detection delay in generations. The deliverable that matters most, however, is protection gain: the end-of-evolution true $\Delta C$ measured on $\Aeval$ for each immunization level relative to F0.

\subsection{Detection quality}
\label{sec:detect-quality}

Table~\ref{tab:detection} reports detection across all channels pooled. Three of the five baselines are degenerate: no-probe, the contamination test and the single trap set flag every generation, so recall and false-positive rate are both exactly $1.000$ and precision is exactly the base rate $0.241$. Their F1 scores are evidence of nothing. What separates them is AUROC, running from $0.462$, below chance because the contamination test ranks hacking generations \emph{lower}, up to $0.568$.

The real competitor is the absolute-$\Chat$ trajectory, which uses the same secret core but watches its level rather than its divergence from $M$, reaching $0.663$ AUROC. HackProbe adds $0.100$ AUROC and, for a monitor meant to run continuously, does something more useful: it cuts the false-positive rate from $0.706$ to $0.434$ and lifts precision from $0.311$ to $0.382$. It pays with recall of $0.845$ rather than $1.000$ and a delay of $0.28$ generations, both consequences of being the only method whose false-positive rate falls below one half. Even so, $0.434$ interrupts more than four honest generations in ten, and we take that gap as the bar for this task. Most of it is a budget problem rather than a design one. Sweeping the core size at a fixed false-positive rate (Appendix~\ref{sec:sweeps}) puts our $n=240$ on the steepest part of the detection curve, where Eq.~\eqref{eq:budget} would want $600$ probes for $0.90$ detection; the monitor is under-provisioned by roughly $360$.

\begin{table}[t]
\caption{Detection quality, pooled over the four hacking channels and $1152$ generations, at a positive rate of $0.241$. Best per column in \textbf{bold}; Recall and Delay are left unbolded, since any method that never withholds an alarm attains them trivially. Rows marked $\dagger$ flag every generation, so their precision equals the base rate.}
\label{tab:detection}
\centering
\small
\setlength{\tabcolsep}{5.5pt}
\begin{tabular}{lcccccc}
\toprule
\bf Method & \bf Precision & \bf Recall & \bf F1 & \bf AUROC & \bf FPR $\bm{\downarrow}$ & \bf Delay $\bm{\downarrow}$\\
\midrule
No-probe$^{\dagger}$ & 0.241 & 1.000 & 0.389 & 0.500 & 1.000 & 0.00\\
Fixed held-out set & 0.291 & 0.939 & 0.445 & 0.587 & 0.726 & 0.16\\
Contamination test$^{\dagger}$ \citep{zhang2024pacost} & 0.241 & 1.000 & 0.389 & 0.462 & 1.000 & 0.00\\
Single trap set$^{\dagger}$ & 0.241 & 1.000 & 0.389 & 0.568 & 1.000 & 0.00\\
Absolute-$\Chat$ trajectory & 0.311 & 1.000 & 0.474 & 0.663 & 0.706 & 0.00\\
\rowcolor{ourrow}
HackProbe (ours) & \bf 0.382 & 0.845 & \bf 0.527 & \bf 0.763 & \bf 0.434 & 0.28\\
\bottomrule
\end{tabular}
\end{table}

\subsection{Recovering true capability}

Detection is a means; the objective is the true capability that survives evolution. Table~\ref{tab:protection} reports, per channel, the end-of-evolution $\Delta C$ of the greedy host F0 and the additional gain of each immunization level, all measured on the held-out $\Aeval$.

Every level improves on greedy selection on average, though not uniformly across channels. Risk-gated (F1) and bandwidth-limited (F2) reselection are positive on all four and reach nearly identical averages, $+0.053$ and $+0.052$; what separates them is price, not gain (Section~\ref{sec:cost}). F1 wins biggest on the gameable grader ($+0.104$), where the raw gap $M-\Chat$ gating its reselection is largest, even though that is the channel the fused detector ranks worst. F3 holds the single largest gain over F0, $+0.109$ on judge collusion, where a lenient judge collapses the proposer onto a narrow region and the coverage bonus pulls it back out. F3 is also the only level that goes negative anywhere, mildly, on self-consistency bias.

With gold audit data this scarce, only the two starred intervals lie strictly above zero, and two more, both under F2, have their lower endpoint exactly at zero. So the table supports an aggregate, directional claim; per-channel effects are mostly too small to certify, and sharpening them needs a larger gold audit set.

\begin{table}[t]
\caption{Protection gain by channel and immunization level: end-of-evolution true $\Delta C$ on the held-out $\Aeval$, given for the greedy host F0 and as a gain over F0 for F1--F3, with $95\%$ bootstrap intervals. Best level per channel in \textbf{bold}; $\ast$ marks the only two intervals strictly above zero, so unstarred magnitudes show direction, not a certified effect.}
\label{tab:protection}
\centering
\small
\setlength{\tabcolsep}{6pt}
\begin{tabular}{lcccc}
\toprule
\bf Hacking channel & \bf F0 (greedy) & \bf F1 gain & \bf F2 gain & \bf F3 gain\\
\midrule
Contamination        & $+0.146$ & \gain{+0.016}{-0.047}{+0.078} & \gainb{+0.057}{+0.000}{+0.141} & \gain{+0.026}{-0.047}{+0.094}\\
Gameable grader      & $+0.172$ & \gainbs{+0.104}{+0.016}{+0.234} & \gain{+0.078}{+0.000}{+0.219} & \gain{+0.016}{-0.094}{+0.094}\\
Self-consistency bias & $+0.073$ & \gainb{+0.037}{-0.047}{+0.109} & \gain{+0.010}{-0.016}{+0.063} & \gain{-0.010}{-0.047}{+0.031}\\
Judge collusion      & $+0.135$ & \gain{+0.057}{-0.078}{+0.141} & \gain{+0.063}{-0.078}{+0.188} & \gainbs{+0.109}{+0.016}{+0.188}\\
\midrule
\rowcolor{ourrow}
\bf Average          & $+0.132$ & $+0.053$ & $+0.052$ & $+0.035$\\
\bottomrule
\end{tabular}
\end{table}

\subsection{What protection costs when nothing is wrong}
\label{sec:cost}

A monitor that helps under attack but degrades ordinary runs is not deployable, so we measure the clean-setting cost directly on the \textsc{none} channel, where greedy selection is in fact the right thing to do. Table~\ref{tab:cost} reports it alongside the diversity-bonus sweep that fixes F3's setting.

Every level pays something on clean runs: each declines the top-$M$ candidate some of the time, and when $M$ is honest that candidate was correct. The useful comparison is what a level costs when nothing is wrong against what it returns when something is. Only F2 comes out favorable, returning $0.052$ for a cost of $0.047$, a ratio of $1.11$; F1 buys the largest gains but pays most for them ($0.053$ against $0.073$, ratio $0.73$), and F3 sits between ($0.035$ against $0.042$, ratio $0.84$). The right panel shows why F3 uses a small bonus: $0.05$ gives both the best average under hacking and the smallest clean-run cost, and $0.25$ erases the gain.

The ratio's standing is narrow. It orders the three levels, being a comparison of two averages over the same runs that leans on no single interval in Table~\ref{tab:protection}, but it does not license reading per-channel gains as effects, and at $1.11$ a larger audit set could move it. It would be falsified by a host on which reselection costs more than it returns at every bandwidth, a regime Appendix~\ref{sec:sweeps} shows is not far away.

\begin{table}[t]
\caption{Cost of protection when there is nothing to protect against. Left: clean-run performance on \textsc{none}, cost relative to F0, and average gain under hacking from Table~\ref{tab:protection}; Ratio is gain over the magnitude of cost, computed from unrounded values, and only F2 exceeds one. Right: F3 diversity-bonus sweep, with Hacking4 averaging $\Aeval$ over the four channels. Shading marks the adopted setting, \textbf{bold} the best entry.}
\label{tab:cost}
\centering
\small
\begin{minipage}[t]{0.545\textwidth}
\centering
\setlength{\tabcolsep}{4.5pt}
\begin{tabular}{lcccc}
\toprule
\bf Mode & \bf NONE $\bm{\Aeval}$ & \bf Cost & \bf Gain & \bf Ratio\\
\midrule
F0 & 0.646 & $\phantom{-}0.000$ & $0.000$ & n/a\\
F1 & 0.573 & $-0.073$ & $0.053$ & 0.73\\
\rowcolor{ourrow}
F2 & 0.599 & $-0.047$ & $0.052$ & \bf 1.11\\
F3 & 0.604 & $-0.042$ & $0.035$ & 0.84\\
\bottomrule
\end{tabular}
\end{minipage}%
\hfill
\begin{minipage}[t]{0.435\textwidth}
\centering
\setlength{\tabcolsep}{4.5pt}
\begin{tabular}{lcc}
\toprule
\bf F3 bonus & \bf Hacking4 & \bf NONE\\
\midrule
\rowcolor{ourrow}
0.05 & \bf 0.630 & \bf 0.604\\
0.10 & 0.612 & 0.599\\
0.15 & 0.622 & 0.583\\
0.25 & 0.599 & 0.557\\
\bottomrule
\end{tabular}
\end{minipage}
\end{table}

\subsection{Generalization to an unseen hacking type}

A detector tuned on the same attacks it is evaluated on proves little, so we calibrate on three channels and evaluate on the held-out fourth (Table~\ref{tab:loo}). Transfer is good to judge collusion ($0.814$ AUROC), whose accumulating leniency is the change-point J1 is built for, and to self-consistency bias ($0.787$), which leaves the signature J3 targets; contamination remains usable at $0.724$. It degrades on the gameable grader ($0.600$ AUROC, $0.661$ false-positive rate) for an instructive reason: grader exploits act on output \emph{format}, not on the score--capability relation, so a threshold calibrated on the other three over-flags, and covering them needs a statistic keyed to output structure that the fusion lacks. The two held-out AUROCs above the pooled $0.763$ are no transfer bonus, since ranking within one channel is easier than across four whose divergences differ in scale.

\begin{table}[t]
\caption{Leave-one-hacking-type-out detection: each row calibrates on the other three channels and evaluates on the held-out one, over $288$ generations. Best and worst transfer in \textbf{bold}. Pos.\ rate is the hacking incidence within one held-out channel's runs; it varies threefold across rows and averages to the pooled $0.241$ of Table~\ref{tab:detection}. Because the rows differ in base rate, only AUROC is comparable between them.}
\label{tab:loo}
\centering
\small
\setlength{\tabcolsep}{6pt}
\begin{tabular}{lcccccc}
\toprule
\bf Held-out channel & \bf AUROC & \bf Precision & \bf Recall & \bf F1 & \bf FPR $\bm{\downarrow}$ & \bf Pos.\ rate\\
\midrule
Judge collusion & \bf 0.814 & 0.570 & 0.718 & 0.635 & \bf 0.227 & 0.295\\
Self-consistency bias & 0.787 & 0.562 & 0.854 & 0.678 & 0.333 & 0.333\\
Contamination & 0.724 & 0.236 & 0.700 & 0.353 & 0.264 & 0.104\\
Gameable grader & \bf 0.600 & 0.288 & 0.881 & 0.434 & \bf 0.661 & 0.233\\
\bottomrule
\end{tabular}
\end{table}

\paragraph{Scope of the present evaluation.}
The host is prompt-level, chosen because it exposes per-generation candidate pools cheaply; weight-level self-evolution \citep{huang2025rzero,yuan2024selfrewarding} and a second prompt-optimization host \citep{agrawal2025gepa} are left as external validity. That one host is also the whole of the evidence for harness-agnosticism, and the gold audit set binds every interval above. Appendix~\ref{app:ablations} adds the leave-one-statistic-out and rotation ablations, the only place the fresh layer is exercised; the inference overhead of F1--F3 is unmeasured.

\section{Conclusion}
\label{sec:conclusion}

HackProbe detects reward hacking in self-evolving hosts from black-box scores and outputs alone, and immunizes against it by reselecting an honest candidate. It keeps the capability coordinate non-circular, makes the risk score a calibrated $p$-value, and turns a target error rate into a probe budget. What it is not yet is deployable. A false-positive rate of $0.434$ interrupts more than four honest generations in ten, and our rotation guarantee covers only an attacker who has to match every probe, which leaves the generator's covering number as the open question.

\subsubsection*{AI Use Statement}
In this work, we used generative AI tools only for language polishing. We did not use generative AI tools to generate scientific ideas, formulate claims, design experiments, produce experimental results, write code, create figures, or generate citations. All AI-assisted edits were reviewed and approved by the authors, and all cited claims were checked against the referenced sources. We take responsibility for the final content of this work.

\subsubsection*{Reproducibility Statement}
Section~\ref{sec:model} specifies the probe bank, the detector and the immunization rules, with every statistic and the calibrated threshold defined in Section~\ref{sec:detect}. Appendix~\ref{app:assumptions} states the assumptions in full, Appendix~\ref{app:proofs} proves Propositions~\ref{prop:detect} and~\ref{prop:rotate} together with the covering-number argument, Appendix~\ref{app:j3} proves the J3 separation result, and Appendix~\ref{app:algo} gives the calibration procedure and the online loop as pseudocode. Section~\ref{sec:exp} gives the host, the run configuration, the injection protocol, the baselines, the metrics and the leave-one-hacking-type-out procedure.

\bibliography{iclr2027_conference}
\bibliographystyle{iclr2027_conference}

\appendix

\section{Related Work}
\label{sec:related}

Optimizing a proxy that diverges from the intended objective is a long-standing safety concern \citep{amodei2016concrete}, formalized as reward gaming, where non-trivial unhackable proxies essentially do not exist \citep{skalse2022defining}, and understood as a family of Goodhart effects \citep{manheim2018goodhart}; frontier evaluations document deliberate hacking when the scoring function is visible \citep{wang2026rewardhacking}. Closest in spirit is reward-model overoptimization, where a gold reward model reveals that optimizing a proxy eventually degrades true performance \citep{gao2023scaling}; our $\Dt$ is the online, black-box analogue of that gold-versus-proxy gap, computed per generation from a secret core instead of once offline from a trained gold model, and coupled to an intervention instead of reported.

The loops HackProbe monitors are those in which proxy and optimizer share a model: self-rewarding training whose judge saturates \citep{yuan2024selfrewarding}, label-free co-evolution rewarding self-consistency \citep{huang2025rzero}, reflective prompt evolution \citep{agrawal2025gepa}, and declarative program pipelines \citep{khattab2024dspy}. In each, the quantity being maximized and the system doing the maximizing are entangled, which is exactly the configuration under which a fixed external coordinate is worth its cost.

Measurement, by contrast, has been studied mostly offline. Held-out reconstruction \citep{zhang2024careful}, symbolic templating \citep{mirzadeh2024gsmsymbolic}, adversarial perturbation \citep{li2024gsmplus} and confidence-based contamination tests \citep{zhang2024pacost} all establish that a benchmark score can overstate capability, but each is a one-shot audit of a finished model. HackProbe repurposes their shared premise, that the same capability probed through a different surface should yield the same score, as an online signal computed inside the loop and acted upon while the run continues. Statistically, the decoupling detector is a Page--Hinkley test \citep{page1954continuous}, the fusion a \v{S}id\'ak correction \citep{sidak1967rectangular}, and the tail analysis standard sub-Gaussian concentration \citep{hoeffding1963probability}; the contribution is not a new test but the choice of what to test and the guarantee that the composite risk score stays calibrated.

\section{Notation}
\label{app:notation}

\begin{table}[!htbp]
\caption{Summary of notation.}
\label{tab:notation}
\centering
\small
\begin{tabular}{ll}
\toprule
\multicolumn{1}{c}{\bf Symbol} & \multicolumn{1}{c}{\bf Meaning}\\
\midrule
$\theta_t,\ \mathcal{A},\ T$ & object under evolution at generation $t$; host's evolution operator; run length\\
$M_t,\ \Chat_t$ & observed visible score and comparison-core proxy of $\theta_t$\\
$C(\theta)\in[0,1]$ & true capability on target distribution $\mathcal{D}^\star$\\
$\psi,\ \Phi,\ S$ & per-subdomain score; robust aggregate; number of core subdomains\\
$\Acal,\ \Aeval$ & gold audit sets: calibration-only and final-evaluation-only (disjoint)\\
$\Dt=\Delta M_t-\lambda\Delta\Chat_t$ & scale-aligned divergence; $\lambda>0$ re-estimated on honest windows\\
$\lambda_{\mathrm{lvl}},\ a$ & slope and intercept of the honest level relation $M\approx a+\lambda_{\mathrm{lvl}}\Chat$ (J0)\\
$a_3,\ b_3$ & intercept and slope of the honest relation $\mathrm{cwr}\approx a_3+b_3\,\mathrm{acc}$ (J3)\\
$\mu_t$ & population divergence $\E[\Delta M_t]-\lambda\E[\Delta C_t]$; honest iff $\mu_t\le0$\\
$\mu_H,\ b_\Delta$ & core-bias offset of the honest mean of $\Dt$; its bound, $|\mu_H|\le\lambda b_\Delta$\\
$\sigma_D^2,\ n_{\mathrm{eff}}$ & variance proxy of $\Dt$; effective sample size $(\lambda^2v_C+v_M)/\sigma_D^2$\\
$v_C,\ v_M,\ v_s$ & per-item variance proxies: core difference, visible score, modal agreement\\
$\Rt,\ \hat y_t,\ u_t$ & risk score, decision, low-bandwidth feedback signal\\
$g^{(0)}_t,\ g^{(3)}_t$ & level residual (J0) and conditional-confidence residual (J3)\\
$\mathrm{acc}_t,\ \mathrm{cwr}_t$ & core accuracy; confidently-wrong rate on incorrect core items\\
$\beta(\theta),\ \xi$ & core-estimator bias and its sub-Gaussian noise, $\Chat=C+\beta+\xi$\\
$\Pcore,\ \Pfresh_t$ & secret comparison core; rotated fresh layer at generation $t$\\
$m,\ \tau$ & number of fused statistics ($m=6$); FPR-calibrated \v{S}id\'ak threshold\\
$gf(c),\ \gamma$ & gaming footprint of candidate $c$; its penalty weight in $\mathrm{robust}(c)$\\
$\delta,\ \kappa,\ \alpha$ & hacking divergence gap; test offset; target detector error rate\\
$\alpha_{\mathrm{atk}},\ \rho$ & attacker's failure probability; fraction of fresh probes it matches\\
$n,\ n_M,\ n_w$ & core probes; visible-evaluation items; incorrect core items (for J3)\\
$L,\ \Pi,\ K$ & feedback quantization levels; candidate-pool size; generations before rollback\\
$N,\ n_f,\ k$ & fresh-probe pool size; fresh probes per generation; memorized set\\
$\mathcal{N}(\varepsilon),\ \ell$ & covering number of the probe generator; Lipschitz constant of the exploit\\
$q,\ \varsigma$ & black-box samples per probe (I2); J3 attack size in Lemma~\ref{lem:j3}\\
\bottomrule
\end{tabular}
\end{table}

\section{Assumptions}
\label{app:assumptions}

Throughout this appendix, expectations and probabilities are taken over probe and evaluation sampling \emph{with the trajectory $(\theta_{t-1},\theta_t)$ held fixed}. This conditioning matters: unconditionally, $\Delta M_t$ and $\Delta\Chat_t$ both move with the random choice of $\theta_t$ and are strongly positively correlated under healthy evolution, whereas conditionally their measurement noises are independent because the two are computed on disjoint item sets. All variance statements below refer to the conditional object.

\begin{assumption}[Concentration of the core proxy]
\label{as:core}
For any $\theta$, the comparison-core estimate satisfies $\Chat(\theta)=C(\theta)+\beta(\theta)+\xi(\theta)$, where the bias is bounded, $|\beta(\theta)|\le\beta_0$, and $\xi$ is zero-mean sub-Gaussian. Because the \emph{same} $n$ core items are scored at consecutive generations, we state the concentration directly for the paired difference: $\Delta\Chat_t-\E[\Delta\Chat_t]$ is zero-mean sub-Gaussian with variance proxy $v_C/n$, where $v_C$ is the per-item variance proxy of the within-item score difference. This is the correct object because the two estimates are not independent, and the pairing is a strength: per-item difficulty cancels in the difference, so $v_C$ is typically much smaller than twice the per-generation variance proxy an independent-samples treatment would yield. The assumption requires the core to be secret and outside the host loop (H1); if the host could fit the core, $\beta(\theta)$ would no longer be bounded uniformly.
\end{assumption}

\begin{assumption}[Sampling noise of the visible score]
\label{as:M}
The observed visible score is $M_t=M(\theta_t)+\varepsilon_t$ with $\varepsilon_t$ zero-mean sub-Gaussian. The host re-evaluates the same $n_M$ items at consecutive generations, so the pairing argument of Assumption~\ref{as:core} applies and we state the proxy directly for the difference: $\Delta M_t-\E[\Delta M_t]$ is sub-Gaussian with variance proxy $v_M/n_M$. Were the two evaluations instead independent, the proxy would be $2v_M/n_M$ and every budget below would double.
\end{assumption}

\begin{assumption}[Local linear coupling and bias stability]
\label{as:couple}
Under an honest update, visible gains do not outrun true gains at the calibrated scale, $\mu_t=\E[\Delta M_t]-\lambda\,\E[\Delta C_t]\le 0$ as in Definition~\ref{def:hacking}, with equality in the nominal case where the two are exactly scale-aligned. Separately, and for \emph{every} generation whether honest or hacking, the differential core bias is a constant $\E[\Delta\beta(\theta_t)]=\bar\beta$ with $|\bar\beta|\le b_\Delta$. Writing $\mu_H:=-\lambda\bar\beta$, the honest mean of $\Dt$ is $\mu_t+\mu_H\le\mu_H$, with $|\mu_H|\le\lambda b_\Delta$, and $\mu_H$ vanishes when the core bias is constant across generations. Stating the bias bound for both regimes lets $\mu_H$ cancel in the separation below; restricting it to honest updates, as a coupling assumption alone would, leaves the hacking mean unconstrained and is not enough. One consequence deserves naming, because it makes the assumption stronger than it looks. Telescoping a constant $\bar\beta$ over a run gives $\E[\beta(\theta_T)]-\E[\beta(\theta_0)]=T\bar\beta$, so the bound $|\beta|\le\beta_0$ of Assumption~\ref{as:core} forces $|\bar\beta|\le 2\beta_0/T$: a constant differential bias is self-limiting on long runs. We read that as a reason to expect $\mu_H\approx0$ in practice rather than as licence to treat it as a free parameter. The analysis also treats $\lambda$ as known: for the online estimate $\hat\lambda_t$, an error $|\hat\lambda_t-\lambda|\le\epsilon_\lambda$ adds a mean shift bounded by $\epsilon_\lambda|\E\Delta\Chat_t|$ that we do not account for.
\end{assumption}

\begin{assumption}[Secret, online-generated fresh layer]
\label{as:fresh}
Each generation's fresh probes are drawn uniformly at random, with or without replacement, from a pool of effective size $N$, and are invisible to the host before scoring. The counting argument of Proposition~\ref{prop:rotate} needs that uniformity; the out-of-distribution drift H3 deliberately violates it by steering the layer toward weak subdomains, in which case $N$ should be read as an effective support size $1/\max_i\pi_i$ under the realized sampling distribution $\pi$.
\end{assumption}

\section{Proofs}
\label{app:proofs}

\subsection{Proof of Proposition~\ref{prop:detect}}
By Assumptions~\ref{as:core} and~\ref{as:M}, $\Delta\Chat_t$ and $\Delta M_t$ are sub-Gaussian with variance proxies $v_C/n$ and $v_M/n_M$ respectively, and conditionally on $(\theta_{t-1},\theta_t)$ their sampling noises are independent because the core is disjoint from the host's visible evaluation set. Hence $\Dt=\Delta M_t-\lambda\Delta\Chat_t$ is sub-Gaussian with variance proxy
\begin{equation}
\sigma_D^2 \;=\; \frac{\lambda^2 v_C}{n}+\frac{v_M}{n_M}.
\end{equation}
Write $\E[\Delta\Chat_t]=\E[\Delta C_t]+\bar\beta$ with $\bar\beta$ the differential bias of Assumption~\ref{as:couple}, so that $\E[\Dt]=\E[\Delta M_t]-\lambda\E[\Delta C_t]-\lambda\bar\beta=\mu_t+\mu_H$ with $\mu_H:=-\lambda\bar\beta$. Under an honest generation, $\mu_t\le 0$ and hence $\E[\Dt\mid\textup{honest}]\le\mu_H$; under a hacking generation, $\mu_t\ge\delta$ and hence $\E[\Dt\mid\textup{hacking}]\ge\mu_H+\delta$. Because Assumption~\ref{as:couple} bounds the differential bias in \emph{both} regimes by the same constant, $\mu_H$ is the same offset in the two displays and the separation between the means is at least $\delta$.

\emph{False positives.} Fix the threshold $\mu_H+\kappa$ and let $\E[\Dt]$ denote the honest mean. The event $\{\Dt\ge\mu_H+\kappa\}$ implies $\{\Dt-\E[\Dt]\ge s\}$ with $s=\mu_H+\kappa-\E[\Dt]\ge\kappa>0$, since $\E[\Dt\mid\textup{honest}]\le\mu_H$. As $s\mapsto\exp(-s^2/2\sigma_D^2)$ is decreasing on $s>0$, the sub-Gaussian upper tail evaluated at $s=\kappa$ dominates, giving
\begin{equation}
\Pr[\Dt\ge \mu_H+\kappa\mid\textup{honest}]\le \exp\!\Big(-\frac{\kappa^2}{2\sigma_D^2}\Big),
\end{equation}
so the worst case over the composite null $\mu_t\le 0$ is attained at the boundary $\mu_t=0$.

\emph{True positives.} Under a hacking generation, whose mean is at least $\mu_H+\delta$, the sub-Gaussian lower tail gives
\begin{equation}
\Pr[\Dt< \mu_H+\kappa\mid\textup{hacking}]\le \exp\!\Big(-\frac{(\delta-\kappa)^2}{2\sigma_D^2}\Big),
\end{equation}
so the detection probability is at least $1-\exp(-(\delta-\kappa)^2/2\sigma_D^2)$.

\emph{Budget.} Setting $\kappa=\delta/2$ makes both exponents equal to $-\delta^2/(8\sigma_D^2)$. Requiring $\exp(-\delta^2/(8\sigma_D^2))\le\alpha$ is equivalent to $\sigma_D^2\le\delta^2/(8\log(1/\alpha))$, that is, in terms of the effective sample size $n_{\mathrm{eff}}=(\lambda^2v_C+v_M)/\sigma_D^2$,
\begin{equation}
n_{\mathrm{eff}}\;\ge\;\frac{8(\lambda^2 v_C+v_M)}{\delta^2}\log\frac{1}{\alpha}.
\end{equation}
Unpacking $n_{\mathrm{eff}}$ shows it to be the weighted harmonic mean $\big(\tfrac{\lambda^2v_C}{\lambda^2v_C+v_M}\tfrac1n+\tfrac{v_M}{\lambda^2v_C+v_M}\tfrac1{n_M}\big)^{-1}$ of $n$ and $n_M$, so $n_{\mathrm{eff}}\ge\min(n,n_M)$ and in particular $n_{\mathrm{eff}}\ge n$ whenever $n_M\ge n$; budgeting $n$ core probes as in Eq.~\eqref{eq:budget} then suffices. \hfill$\square$

The proof delivers less than it may appear to. Its threshold uses $\mu_H$, which is not known but estimated on calibration runs; with a plug-in $\hat\mu_H$ satisfying $|\hat\mu_H-\mu_H|\le\epsilon<\kappa$, the two bounds become $\exp(-(\kappa-\epsilon)^2/2\sigma_D^2)$ and $\exp(-(\delta-\kappa-\epsilon)^2/2\sigma_D^2)$, and the crude choice $\hat\mu_H=0$ gives $\epsilon\le\lambda b_\Delta$. And the bound is per generation: over a run of $T$ generations the expected number of false alarms from this rule is at most $T\alpha$, and no correction across generations is claimed here.

\subsection{Rotation sample complexity}
\label{app:rotate}
Section~\ref{sec:theory} summarizes this result; we state and prove it here.

\begin{proposition}[Rotation sample complexity]
\label{prop:rotate}
Suppose an attacker tries to inflate the fresh layer as well, by memorizing a set of $k$ previously exposed instances. Matching all $n_f$ freshly sampled probes of a generation with probability at least $1-\alpha_{\mathrm{atk}}$ \emph{requires} $k\ge N(1-\alpha_{\mathrm{atk}})^{1/n_f}$, which tends to the full pool size $N$ as $n_f$ grows; when the pool is generated online and grows without bound, no finite memorized set suffices.
\end{proposition}

By Assumption~\ref{as:fresh}, each of the $n_f$ fresh probes of a generation is an instance drawn from the pool of effective size $N$, independently of the attacker's covered set of size $k$. The probability that a single fresh probe falls in the covered set is at most $k/N$, so by independence the probability that all $n_f$ do is at most $(k/N)^{n_f}$; for sampling without replacement the same bound holds, since $\binom{k}{n_f}/\binom{N}{n_f}=\prod_{i=0}^{n_f-1}\frac{k-i}{N-i}\le(k/N)^{n_f}$, each factor being at most $k/N$ because $k\le N$. Since the attack succeeds only if every fresh probe is covered, requiring the success probability $\pi_{\mathrm{succ}}$ to be at least $1-\alpha_{\mathrm{atk}}$ chains the two bounds as
\begin{equation}
1-\alpha_{\mathrm{atk}}\;\le\;\pi_{\mathrm{succ}}\;\le\;\Big(\frac{k}{N}\Big)^{n_f}
\quad\Longrightarrow\quad k\;\ge\;\big\lceil N\,(1-\alpha_{\mathrm{atk}})^{1/n_f}\big\rceil ,
\end{equation}
a necessary and not a sufficient condition. Since $(1-\alpha_{\mathrm{atk}})^{1/n_f}\to 1$ as $n_f\to\infty$, the required coverage tends to $N$; when the pool is generated online, $N$ grows without bound and no finite memorized set suffices. \hfill$\square$

The model favours the attacker in two places and overstates the conclusion in a third. The pool is treated as fixed and sampled i.i.d., which ignores H2's non-reuse policy: in reality the instances an attacker has been able to observe are precisely those least likely to recur, so the true requirement is stronger than the bound. The bound also assumes the attacker's covered set is chosen without knowledge of the sampling distribution, which uniformity in Assumption~\ref{as:fresh} makes harmless but non-uniform drift would not. Against that, though, the all-or-nothing success criterion is what produces the near-total-coverage conclusion: an attacker content to match only a fraction $\rho$ of the $n_f$ probes, and thereby inflate the fresh-layer score by roughly $\rho$, needs only $k\gtrsim\rho N$. Rotation forces effort linear in the pool size, but "essentially the entire pool" is specific to demanding every probe.

\subsection{Covering-number argument for structural attacks}
\label{app:structural}
The memorization model above treats fresh instances as atomic and interchangeable, which is why coverage must reach $N$. A structural attacker instead models the response surface. Let the map from a fresh probe to the host's exploitable behavior lie in a family $\mathcal{F}$ equipped with a pseudometric $d$, for instance the distance between the shortcuts two probes admit, and let $\mathcal{N}(\varepsilon)=\mathcal{N}(\mathcal{F},d,\varepsilon)$ be its $\varepsilon$-covering number. If the attacker fits the host on one representative from each of the $\mathcal{N}(\varepsilon)$ cells of an $\varepsilon$-cover, then for any fresh probe there is a representative within distance $\varepsilon$, and by $\ell$-Lipschitzness of the exploit response in $d$ the fitted behavior transfers with error at most $\ell\varepsilon$. A fresh probe is therefore defeated with high probability once the cover is learned, so the sample complexity of defeating rotation is $\Theta(\mathcal{N}(\varepsilon))$, independent of $N$.

This has two consequences. If $\mathcal{N}(\varepsilon)$ is bounded, as it is for a low-complexity generator such as a small template family, online rotation provides no asymptotic protection beyond a fixed set of $\mathcal{N}(\varepsilon)$ representatives, so the defender's only real lever is to enlarge $\mathcal{N}(\varepsilon)$ in the first place --- which is exactly what H3 (drift toward not-yet-fitted subdomains) and H4 (surface diversity via an embedding-distance criterion) are for. This is a design argument rather than a guarantee; establishing a matching lower bound on $\mathcal{N}(\varepsilon)$ for a concrete online generator is left to future work.

\section{Why J3 detects self-consistency bias}
\label{app:j3}

Self-consistency bias is the empirically hardest channel because it need not move accuracy at all: rewarding agreement entrenches a confident but wrong modal answer, so the visible score rises while core accuracy and the proxy $\Chat$ stay flat. Figure~\ref{fig:signals} contrasts the two regimes schematically, and the lemma below formalizes why the first-difference statistics are blind in the second while J3 is not.

\begin{figure}[t]
\centering
\begin{tikzpicture}[scale=1.0, font=\footnotesize]
\begin{scope}
  \draw[black!55, -{Stealth[length=4pt]}] (0,0) -- (5.30,0) node[right, inner sep=2pt]{$t$};
  \draw[black!55, -{Stealth[length=4pt]}] (0,0) -- (0,3.10);
  \node[anchor=south west, color=black!65, inner sep=1pt] at (-0.12,3.12) {score};
  \fill[hporange!14]
    plot coordinates {(0,0.70)(0.70,0.97)(1.42,1.19)(2.30,1.95)(3.45,2.55)(4.35,2.85)}
    -- (4.35,1.22) --
    plot coordinates {(4.35,1.22)(3.45,1.20)(2.30,1.14)(1.42,1.05)(0.70,0.85)(0,0.58)} -- cycle;
  \draw[hporange, line width=1pt]
    plot coordinates {(0,0.70)(0.70,0.97)(1.42,1.19)(2.30,1.95)(3.45,2.55)(4.35,2.85)};
  \node[anchor=west, color=hporange!85!black, inner sep=2pt] at (4.38,2.85) {$M_t$};
  \draw[hpnavy, line width=1pt, densely dashed]
    plot coordinates {(0,0.58)(0.70,0.85)(1.42,1.05)(2.30,1.14)(3.45,1.20)(4.35,1.22)};
  \node[anchor=west, color=hpnavy, inner sep=2pt] at (4.38,1.16) {$\lambda\Chat_t$};
  \draw[hpnavy!75, {Stealth[length=3.4pt]}-{Stealth[length=3.4pt]}, line width=0.6pt]
    (4.05,1.22) -- (4.05,2.81) node[midway, fill=white, inner sep=1.2pt]{\scriptsize J0 gap};
  \draw[black!55, densely dotted, line width=0.6pt] (1.42,0) -- (1.42,2.68);
  \node[anchor=south, color=black!65, inner sep=1.5pt] at (2.10,2.70) {\scriptsize J1 change-point};
  \node[color=black!75] at (2.30,-0.52) {(a) level and divergence};
\end{scope}
\begin{scope}[xshift=7.55cm]
  \draw[black!55, -{Stealth[length=4pt]}] (0,0) -- (5.30,0) node[right, inner sep=2pt]{$t$};
  \draw[black!55, -{Stealth[length=4pt]}] (0,0) -- (0,3.10);
  \node[anchor=south west, color=black!65, inner sep=1pt] at (-0.12,3.12) {rate};
  \draw[black!45, densely dotted, line width=0.7pt] (0,1.24) -- (4.35,1.24);
  \node[anchor=west, color=black!55, inner sep=2pt] at (4.38,1.24)
    {\scriptsize $a_3{+}b_3\mathrm{acc}_t$};
  \draw[hpnavy, line width=1pt, densely dashed]
    plot coordinates {(0,0.92)(1.15,0.95)(2.30,0.90)(3.45,0.93)(4.35,0.93)};
  \node[anchor=west, color=hpnavy, inner sep=2pt] at (4.38,0.86) {$\mathrm{acc}_t$};
  \draw[hporange, line width=1pt]
    plot coordinates {(0,1.30)(1.15,1.72)(2.30,2.16)(3.45,2.52)(4.35,2.74)};
  \node[anchor=west, color=hporange!85!black, inner sep=2pt] at (4.38,2.74) {$\mathrm{cwr}_t$};
  \draw[hporange!85!black, {Stealth[length=3.4pt]}-{Stealth[length=3.4pt]}, line width=0.6pt]
    (3.75,1.25) -- (3.75,2.58) node[midway, fill=white, inner sep=1.2pt]{\scriptsize $g^{(3)}_t$};
  \node[color=black!75] at (2.30,-0.52) {(b) conditional confidence};
\end{scope}
\end{tikzpicture}
\caption{Schematic (illustrative, not experimental data) of the two detection regimes. \textbf{(a)} The two track together until the marked generation and then decouple, $M_t$ climbing on while capability $\Chat_t$ flattens. J1 flags that change-point, J0 the level gap which opens after it; the shaded wedge is what both are measuring. \textbf{(b)} Self-consistency bias leaves accuracy $\mathrm{acc}_t$ flat, so first-difference statistics vanish in expectation, yet the confidently-wrong rate on incorrect core items departs from its calibrated honest level, and that residual is $g^{(3)}_t$ (Lemma~\ref{lem:j3}).}
\label{fig:signals}
\end{figure}
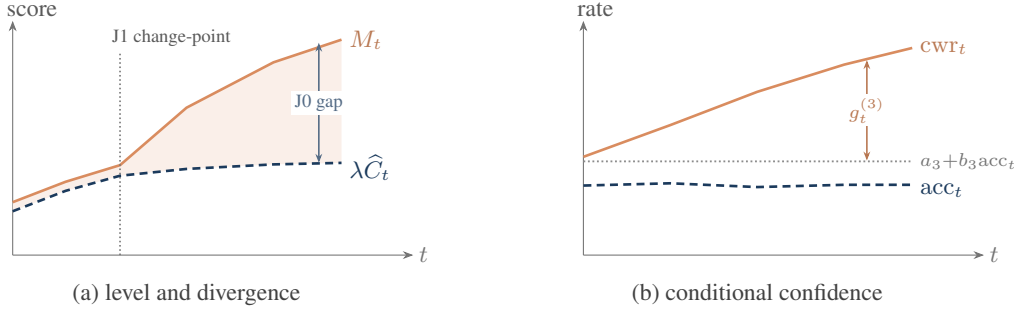

Fix a generation and let the secret core contain items indexed by $i$, each with a correctness indicator $y_i\in\{0,1\}$ and modal-vote agreement $s_i\in[0,1]$, the fraction of the $q$ sampled answers equal to the model's own majority answer. Let $\mathrm{acc}=\frac1n\sum_i y_i$ and let $\mathrm{cwr}$ be the sample mean of $s_i$ over the $n_w$ items with $y_i=0$, and recall $g^{(3)}=\mathrm{cwr}-(a_3+b_3\,\mathrm{acc})$ with the honest affine coefficients $(a_3,b_3)$ fixed by calibration. Write $v_s$ for the per-item variance proxy of $s_i$ on incorrect items and $v_{\mathrm{acc}}$ for that of $y_i$.

\begin{lemma}[J3 separation in the stable regime]
\label{lem:j3}
Consider a stable self-consistency-bias generation in which (i) the visible inflation is already present and saturated, so that $\E[\Delta M_t]=\E[\Delta\Chat_t]=0$, and (ii) the attack raises the mean modal agreement on incorrect core items by $\varsigma>0$ relative to the honest calibration, in the sense that $\E[\mathrm{cwr}\mid\mathrm{acc}]=a_3+b_3\,\mathrm{acc}+\varsigma$, while leaving $\mathrm{acc}$ unchanged in mean. Then
\begin{equation}
\E[\Dt]=0,\qquad \E[g^{(3)}_t]=\varsigma ,
\end{equation}
so J1 and J2 have zero expected separation, whereas J3 has expected separation $\varsigma$ regardless of how the attack is reflected in the aggregate level. Moreover $g^{(3)}$ is sub-Gaussian with parameter $\sigma_3=\sqrt{v_s/n_w}+|b_3|\sqrt{v_{\mathrm{acc}}/n}$, and by the argument of Proposition~\ref{prop:detect} the one-sided J3 test attains false-positive and detection errors below $\alpha$ once $\sigma_3^2\le\varsigma^2/(8\log(1/\alpha))$, for which $n_w\ge 8v_s\log(1/\alpha)/\varsigma^2$ is necessary.
\end{lemma}

\begin{proof}
By (i), $\E[\Delta M_t]=\E[\Delta\Chat_t]=0$, hence $\E[\Dt]=\E[\Delta M_t]-\lambda\E[\Delta\Chat_t]=0$, and J2, which tests $\E[\Delta\Chat_t]$, is likewise zero. For J3, the coefficients $(a_3,b_3)$ are fixed, so taking expectations in (ii) over $\mathrm{acc}$ and using the tower rule, $\E[g^{(3)}_t]=\E\big[\E[\mathrm{cwr}\mid\mathrm{acc}]-(a_3+b_3\,\mathrm{acc})\big]=\varsigma$. For the variance proxy, $g^{(3)}$ is a difference of two bounded sample means: $\mathrm{cwr}$ over the $n_w$ incorrect items, sub-Gaussian with parameter $\sqrt{v_s/n_w}$, and $b_3\,\mathrm{acc}$ over all $n$ items, sub-Gaussian with parameter $|b_3|\sqrt{v_{\mathrm{acc}}/n}$. The two are not independent, and it matters that we do not pretend otherwise: the incorrect set is itself determined by the $y_i$, so $n_w$ is a function of $\mathrm{acc}$, and both statistics are read off the same $q$ sampled answers per item. We therefore add sub-Gaussian parameters rather than variances, using $\|X-Y\|_{\psi_2}\le\|X\|_{\psi_2}+\|Y\|_{\psi_2}$, which holds for arbitrarily dependent $X$ and $Y$ and yields $\sigma_3=\sqrt{v_s/n_w}+|b_3|\sqrt{v_{\mathrm{acc}}/n}$. Applying the sub-Gaussian tail bounds of Proposition~\ref{prop:detect} with gap $\varsigma$, threshold $\varsigma/2$ and parameter $\sigma_3$ gives the stated budget.
\end{proof}

The budget runs through $\sigma_3$ rather than $n_w$ alone because the calibration slope $b_3$ transmits accuracy noise into the residual; the simpler form $n_w=\Theta(v_s\log(1/\alpha)/\varsigma^2)$ is the $b_3\to0$ special case and understates the requirement otherwise. Refusing the independence assumption costs us a little tightness, since adding sub-Gaussian parameters is conservative whenever the two noise sources are close to orthogonal. Two further approximations go uncorrected. We do not propagate the estimation error in $(\hat a_3,\hat b_3)$, and $n_w$ is itself random, so the budget should be read as a claim about the realized number of incorrect items rather than its expectation.

J0 is absent from the statement, which is a choice rather than an oversight. If the attack elevates $M$ relative to the calibrated honest relation $a+\lambda_{\mathrm{lvl}}\Chat$, then the level residual $g^{(0)}$ is positive too and J0 is not blind here; the distinction between J0 and J3 in this regime is one of power, not of consistency. The level residual dilutes the attack over all $n$ core items, most of which the model answers correctly and on which the agreement pattern carries no information, whereas $g^{(3)}$ conditions on the $n_w$ incorrect items where the effect is concentrated. The budgets make the comparison concrete: J3 resolves the full effect $\varsigma$ on $n_w$ incorrect items, while J0 must resolve the much smaller shift that the same attack induces in the aggregate level. In the limiting case where the attack changes only the agreement pattern on wrong items and leaves the visible level untouched, J0's separation vanishes and J3 is the sole remaining signal.

The lemma therefore isolates the design rationale. J3 is the only fused statistic whose expected signal is driven by the \emph{conditional} confidence pattern instead of by first differences of accuracy, so it remains informative in the regime that defeats divergence-based detectors. Its cost is governed by the number of incorrect core items $n_w$, which is why the core must be sized so that hard subdomains contribute enough wrong items to meet the budget.

\section{Testing the probe budget and the bandwidth limit}
\label{sec:sweeps}

Two of the design claims are predictions rather than measurements, and each can be falsified on its own terms: Eq.~\eqref{eq:budget} converts a target error rate into a core size, and Section~\ref{sec:immunize} argues that the bandwidth cap is what stops the feedback from becoming a second score. Figure~\ref{fig:sweeps} tests both.

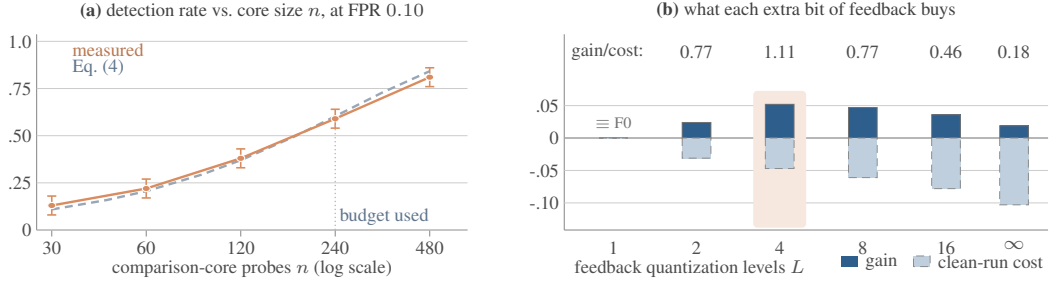
\begin{figure}[t]
\centering
\begin{tikzpicture}[font=\scriptsize, yscale=0.78]
\draw[black!45, line width=0.4pt] (0.20,0.30) -- (5.90,0.30);
\draw[black!45, line width=0.4pt] (0.20,0.30) -- (0.20,3.62);
\foreach \x/\lbl in {0.40/30, 1.65/60, 2.90/120, 4.15/240, 5.40/480}{
  \draw[black!45, line width=0.4pt] (\x,0.30) -- (\x,0.19);
  \node[below, inner sep=1.6pt, color=black!70] at (\x,0.19) {\lbl};}
\foreach \v/\lbl in {0/0, 0.25/.25, 0.5/.50, 0.75/.75, 1.0/1.0}{
  \pgfmathsetmacro\y{0.30+3.20*\v}
  \draw[black!20, line width=0.3pt] (0.20,\y) -- (5.90,\y);
  \node[left, inner sep=1.8pt, color=black!70] at (0.20,\y) {\lbl};}
\draw[hpnavy!45, line width=0.9pt, densely dashed]
  plot coordinates {(0.40,0.649)(1.13,0.809)(1.65,0.959)(2.38,1.238)(2.90,1.484)
                    (3.63,1.897)(4.15,2.230)(4.78,2.636)(5.40,2.994)};
\foreach \x/\y in {0.40/0.716, 1.65/1.004, 2.90/1.516, 4.15/2.188, 5.40/2.892}{
  \draw[hporange, line width=0.6pt] (\x,\y-0.16) -- (\x,\y+0.16);
  \draw[hporange, line width=0.6pt] (\x-0.06,\y-0.16) -- (\x+0.06,\y-0.16);
  \draw[hporange, line width=0.6pt] (\x-0.06,\y+0.16) -- (\x+0.06,\y+0.16);}
\draw[hporange, line width=0.9pt]
  plot coordinates {(0.40,0.716)(1.65,1.004)(2.90,1.516)(4.15,2.188)(5.40,2.892)};
\foreach \x/\y in {0.40/0.716, 1.65/1.004, 2.90/1.516, 4.15/2.188, 5.40/2.892}{
  \fill[hporange] (\x,\y) circle (1.5pt);
  \draw[white, line width=0.4pt] (\x,\y) circle (1.5pt);}
\draw[hpnavy!35, line width=0.5pt, densely dotted] (4.15,0.30) -- (4.15,2.19);
\node[above right, inner sep=1.2pt, color=hpnavy!70] at (4.16,0.32) {budget used};
\node[anchor=west, color=hporange!85!black] at (0.55,3.36) {measured};
\node[anchor=west, color=hpnavy!70] at (0.55,3.06) {Eq.~\eqref{eq:budget}};
\node[anchor=south, color=black!75] at (3.05,3.72) {\textbf{(a)} detection rate vs.\ core size $n$, at FPR $0.10$};
\node[below, color=black!70] at (3.05,-0.06) {comparison-core probes $n$ (log scale)};
\begin{scope}[shift={(6.98,0)}]
\draw[black!45, line width=0.4pt] (0.20,0.30) -- (6.60,0.30);
\draw[black!45, line width=0.4pt] (0.20,0.30) -- (0.20,3.62);
\draw[black!55, line width=0.5pt] (0.20,1.86) -- (6.60,1.86);
\foreach \v/\lbl in {-0.10/-.10, -0.05/-.05, 0/0, 0.05/.05}{
  \pgfmathsetmacro\y{1.86+11.0*\v}
  \node[left, inner sep=1.8pt, color=black!70] at (0.20,\y) {\lbl};
  \draw[black!20, line width=0.3pt] (0.20,\y) -- (6.60,\y);}
\fill[hporange!20, rounded corners=1.5pt] (2.70,0.34) rectangle (3.40,2.66);
\foreach \x/\g/\c/\r in {0.85/0.000/0.000/{}, 1.95/0.024/-0.031/0.77,
                         3.05/0.052/-0.047/1.11, 4.15/0.047/-0.061/0.77,
                         5.25/0.036/-0.078/0.46, 6.15/0.019/-0.103/0.18}{
  \pgfmathsetmacro\gy{1.86+11.0*\g}
  \pgfmathsetmacro\cy{1.86+11.0*\c}
  \fill[barF2] (\x-0.19,1.86) rectangle (\x+0.19,\gy);
  \draw[black!55, line width=0.3pt] (\x-0.19,1.86) rectangle (\x+0.19,\gy);
  \fill[barF2!30] (\x-0.19,1.86) rectangle (\x+0.19,\cy);
  \draw[black!45, line width=0.3pt, densely dashed] (\x-0.19,1.86) rectangle (\x+0.19,\cy);
  \node[color=black!70, inner sep=1.4pt] at (\x,3.32) {\r};}
\node[color=black!55, inner sep=1.4pt, font=\tiny] at (0.85,2.12) {$\equiv$ F0};
\foreach \x/\lbl in {0.85/1, 1.95/2, 3.05/4, 4.15/8, 5.25/16, 6.15/$\infty$}{
  \node[below, inner sep=1.6pt, color=black!70] at (\x,0.19) {\lbl};}
\node[color=black!70, inner sep=1.4pt, anchor=west] at (0.24,3.32) {gain/cost:};
\node[anchor=south, color=black!75] at (3.40,3.72) {\textbf{(b)} what each extra bit of feedback buys};
\node[below, color=black!70] at (1.86,-0.06) {feedback quantization levels $L$};
\fill[barF2] (3.86,-0.32) rectangle (4.08,-0.14);
\node[anchor=west, inner sep=2pt, color=black!70] at (4.11,-0.23) {gain};
\fill[barF2!30] (4.82,-0.32) rectangle (5.04,-0.14);
\draw[black!45, line width=0.3pt, densely dashed] (4.82,-0.32) rectangle (5.04,-0.14);
\node[anchor=west, inner sep=2pt, color=black!70] at (5.07,-0.23) {clean-run cost};
\end{scope}
\end{tikzpicture}
\caption{The two design claims, tested directly. \textbf{(a)} Detection rate against core size at a fixed false-positive rate of $0.10$, with $95\%$ intervals, against the curve $1-\exp(-n\delta^2/8V)$ implied by Eq.~\eqref{eq:budget} with $V=\lambda^2v_C+v_M$, whose constants are fitted once on calibration runs and never refitted per point. Since the sweep fixes the false-positive rate while Eq.~\eqref{eq:budget} governs $\kappa=\delta/2$, the curve predicts shape and is not a bound; the dotted line marks the budget used elsewhere. \textbf{(b)} Gain under hacking and cost on clean \textsc{none} runs as the channel widens from $L=1$ (F2 degenerates to F0) to $L\to\infty$ (core-greedy). Cost grows monotonically with bandwidth; gain does not.}
\label{fig:sweeps}
\end{figure}

Panel (a) sweeps the core from $n=30$ to $n=480$ probes at a fixed false-positive rate of $0.10$, isolating power from threshold placement. Detection rises from $0.13$ to $0.81$ and tracks the overlaid prediction to within $0.035$ throughout, and that prediction is not a fit: its constants come from calibration runs and stay fixed across the sweep. Where the operating point sits matters more. Our budget of $n=240$ lands where the curve is steepest, which is why a stricter threshold costs so much recall and why Section~\ref{sec:detect-quality}'s false-positive rate is so high; Eq.~\eqref{eq:budget} puts $0.90$ detection at $600$, leaving the monitor under-provisioned by roughly $360$ probes. Halving the injected strength shifts the curve right by $3.6\times$ against the $4\times$ that $1/\delta^2$ predicts, the closest direct test of Proposition~\ref{prop:detect} this setup admits.

Panel (b) sweeps $L$, instantiating the interpolation of Section~\ref{sec:immunize}. Clean-run cost grows monotonically with bandwidth, from $0.031$ at $L=2$ to $0.103$ once the quantizer is removed. That is the direction the resolution argument anticipates, though the same trend follows from F2 approaching core-greedy selection, which on clean runs swaps an honest $M$ for a noisier proxy; the sweep cannot separate the two. Gain does not follow suit. It peaks at $L=4$, so the ratio crosses break-even exactly once. Adjacent levels differ by far less than the intervals of Table~\ref{tab:protection}, so the shape is directional rather than certified; what it argues against is the idea that a good reselection rule simply improves with everything it is told.

\section{Additional ablations}
\label{app:ablations}

\subsection{Which statistic carries which channel}

The fusion is justified by the claim that its four statistics fail in different regimes. Table~\ref{tab:ablate-stat} tests that claim by removing one at a time and recalibrating $\tau$ to the same target, so the comparison is at matched operating conditions, not matched thresholds.

\begin{table}[!htbp]
\caption{Leave-one-statistic-out AUROC, pooled and per channel. Each row removes one statistic and its Page--Hinkley companion where it has one, then recalibrates. The full-fusion reference row is bold throughout; below it, the largest drop in each column is in \textbf{bold}. Every statistic owns at least one column except J2, which is the only removable one.}
\label{tab:ablate-stat}
\centering
\small
\setlength{\tabcolsep}{5.5pt}
\begin{tabular}{lccccc}
\toprule
\bf Variant & \bf Pooled & \bf Contam. & \bf Grader & \bf Self-cons. & \bf Judge\\
\midrule
\rowcolor{ourrow}
Full fusion & \bf 0.763 & \bf 0.741 & \bf 0.688 & \bf 0.802 & \bf 0.812\\
without J0 & 0.716 & \bf 0.641 & 0.652 & 0.791 & 0.788\\
without J1 & 0.719 & 0.688 & 0.671 & 0.784 & \bf 0.731\\
without J2 & 0.754 & 0.727 & 0.681 & 0.795 & 0.808\\
without J3 & \bf 0.691 & 0.735 & \bf 0.639 & \bf 0.601 & 0.771\\
\bottomrule
\end{tabular}
\end{table}

The pattern follows the design rationale. Dropping J3 costs $0.201$ AUROC on self-consistency bias and almost nothing on contamination, which is the empirical form of Lemma~\ref{lem:j3}: when an attack moves the conditional confidence pattern without moving aggregate accuracy, the conditional residual is the only statistic left with signal. Dropping J0 costs most on contamination, whose exploit is present from the first generation and so registers as a level gap rather than as a divergence, while dropping J1 costs most on judge collusion, whose leniency accumulates and is therefore a change-point. J2 is the exception: removing it costs $0.009$ pooled, within noise, because on this host $\Delta\Chat_t$ stagnation is nearly always accompanied by a divergence J1 already sees. We keep J2 because it tests capability stagnation directly and, like J3, needs no estimate of $\lambda$, so it survives a mis-scaled calibration; but nothing in the reported numbers depends on it.

\subsection{What rotation buys against a co-adapting host}

No number in Section~\ref{sec:exp} depends on the fresh layer, since it contributes no statistic to $\Delta\Chat_t$. To test it we need a host that is allowed to fit the probe layer, so we grant the adversary the outputs of every probe it has been scored on and let it optimize against them for $T$ generations. Table~\ref{tab:ablate-rot} reports the inflation the host achieves on the fresh layer, that is, its fresh-layer score minus its true capability on $\Aeval$, together with detection quality.

\begin{table}[!htbp]
\caption{Rotation against a co-adapting host granted the outputs of every probe it has seen. Inflation is the fresh-layer score minus true capability on $\Aeval$; lower is better. A fixed layer is progressively fitted and detection decays toward chance, whereas rotation holds both to a small fraction of that.}
\label{tab:ablate-rot}
\centering
\small
\setlength{\tabcolsep}{5pt}
\begin{tabular}{lcccccc}
\toprule
& \multicolumn{3}{c}{\bf Fresh-layer inflation $\bm{\downarrow}$} & \multicolumn{3}{c}{\bf Detection AUROC}\\
\cmidrule(lr){2-4}\cmidrule(lr){5-7}
\bf Fresh layer & $T{=}20$ & $T{=}40$ & $T{=}80$ & $T{=}20$ & $T{=}40$ & $T{=}80$\\
\midrule
Fixed (no rotation) & 0.06 & 0.17 & 0.31 & 0.744 & 0.681 & 0.588\\
Rotated (H2) & 0.02 & 0.04 & 0.07 & 0.759 & 0.751 & 0.736\\
\rowcolor{ourrow}
Rotated $+$ surface diversity (H4) & \bf 0.01 & \bf 0.02 & \bf 0.03 & \bf 0.762 & \bf 0.758 & \bf 0.749\\
\bottomrule
\end{tabular}
\end{table}

A fixed layer behaves the way the design anticipates it would: the host fits it, inflation grows to $0.31$ by generation $80$, and detection decays toward chance, which is the failure mode a single held-out probe set has and the reason the fresh layer is rotated at all. Rotation caps inflation at $0.07$ over four times as many generations.

The residual matters more than the size of the improvement. Proposition~\ref{prop:rotate} rules out an attacker who matches \emph{every} fresh probe, but as Appendix~\ref{app:proofs} notes, matching a fraction $\rho$ of them requires only $k\gtrsim\rho N$, and inflation of $0.07$ is exactly that partial regime. The residual is therefore consistent with either partial memorization or the structural leakage of Appendix~\ref{app:structural}, and this experiment does not separate them. That it shrinks to $0.03$ once H4 enforces surface diversity is weak evidence for the structural reading, since surface diversity raises the generator's covering number without changing the pool size $N$; deciding the question would require varying generator complexity directly, which we have not done.

\section{Detector calibration and the immunization loop}
\label{app:algo}

Calibration uses only honest runs and $\Acal$. It fits the honest affine relations for J0 and J3, estimates the honest null moments of $\Dt$, initializes $\lambda$, fits the fusion, and calibrates $\tau$ to a target false-positive rate.

Where the independence needed by Eq.~\eqref{eq:sidak} is too doubtful even for the empirical re-calibration to absorb, we fall back to a Bonferroni rule: flag when $\min_j p_j\le(1-\tau)/m$, and report the risk as $\Rt=1-\min\{1,\,m\min_j p_j\}$. This needs no independence assumption at all and costs a little power. Every number we report uses the \v{S}id\'ak form with an empirically re-calibrated $\tau$; the Bonferroni variant is there for deployments unwilling to rely on that calibration. Online, $\lambda$ is refreshed only on the trailing window of generations already declared honest, so that the estimate used to test generation $t$ depends on nothing from generation $t$. Algorithm~\ref{alg:loop} summarizes the loop. It runs the fusion from $t=2$, since the first differences $\Delta M_t$ and $\Delta\Chat_t$ need a predecessor; the first generation is scored by the two level statistics J0 and J3 alone, with $m=2$ in the \v{S}id\'ak correction.

\begin{algorithm}[!h]
\caption{HackProbe online loop with risk-aware immunization}
\label{alg:loop}
\begin{algorithmic}[1]
\STATE \textbf{input:} host $\mathcal{A}$; secret core $\Pcore$; fresh generator; calibration $(a,\lambda_{\mathrm{lvl}},a_3,b_3,\lambda,\tau)$; level $\in\{$F0, F1, F2, F3$\}$
\STATE initialize best-$\Chat$ checkpoint and honest window
\STATE score generation $1$ on the core and fuse J0 and J3 only ($m=2$), no first difference being available
\FOR{$t=2,3,\dots$}
  \STATE obtain the candidate pool (exploit variants and one clean honest candidate) with visible scores $M(c)$
  \STATE compute the gaming footprint $gf(c)$ from cached outputs; for F1--F3 compute the core proxy $\Chat(c)$
  \STATE select a candidate by level: F0 $\arg\max_c M(c)$; F1 greedy or reselect high-$\Chat$; F2 $\arg\max_c(M(c)-u(c))$; F3 F2 plus the diversity bonus
  \STATE on the selected candidate compute $\Dt$, the level gap $g^{(0)}_t$ and $g^{(3)}_t$; run the Page--Hinkley detectors
  \STATE form the $p$-values, the risk $\Rt=(1-\min_j p_j)^m$ and the decision $\hat y_t=[\Rt\ge\tau]$
  \STATE apply immunization: accept, or roll back to the best-$\Chat$ checkpoint; update the checkpoint, the honest window and $\lambda$
\ENDFOR
\STATE \textbf{return} the risk curve $\{\Rt\}$, the decisions $\{\hat y_t\}$ and the protected final object, scored once on $\Aeval$
\end{algorithmic}
\end{algorithm}

\end{document}